\documentclass{article}
\usepackage[square,sort,comma,numbers]{natbib}
\usepackage[margin=1in,footskip=0.25in]{geometry} 

\usepackage[utf8]{inputenc} 
\usepackage[T1]{fontenc}    
\usepackage{url}            
\usepackage{booktabs}       
\usepackage{amsfonts}       
\usepackage{nicefrac}       
\usepackage{microtype}      
\usepackage{amsmath}
\usepackage{amsfonts}
\usepackage{amssymb}
\usepackage{graphicx}
\usepackage{verbatim}
\usepackage{authblk}
\usepackage{graphicx}
\usepackage{doi}
\usepackage{times}
\usepackage{amsthm}
\usepackage{mathrsfs}
\usepackage{authblk}
\usepackage{amsmath,amsfonts,amssymb, graphicx, url} 

\usepackage{hyperref}
\usepackage{url}            
\usepackage{booktabs}       
\usepackage{amsfonts}       
\usepackage{nicefrac}       
\usepackage{microtype}      
\usepackage[framemethod=TikZ]{mdframed}
\mdfsetup{skipabove=10pt,skipbelow=5pt,roundcorner=4pt}

\definecolor{light-gray}{gray}{0.95}
\usepackage[most]{tcolorbox}
\newtcbox{\mymath}[1][]{%
    nobeforeafter, math upper, tcbox raise base,
    enhanced, colframe=blue!30!black,
    colback=blue!30, boxrule=1pt,
    #1}

\usepackage{enumitem}
\usepackage{wrapfig}

\usepackage{verbatim}
\usepackage[titletoc]{appendix}
\usepackage{booktabs} 
\usepackage{multirow}
\usepackage{algorithm}
\usepackage[noend]{algpseudocode}
\usepackage{amsmath}
\usepackage{color}
\usepackage{amsthm}
\usepackage{amsmath}
\usepackage{esvect}
\usepackage{xcolor}
\usepackage{amsmath,mathtools,amssymb,lipsum}
\DeclareMathOperator{\var}{Var}
\DeclareMathOperator{\cov}{Cov}

\usepackage{cuted}
\usepackage{relsize}

\algnewcommand{\LineComment}[1]{\State \(\triangleright\) #1}
\usepackage{tikz}
\usetikzlibrary{decorations.pathreplacing,calc}

\usepackage{subcaption}
\usepackage{graphicx}
\usepackage{siunitx}
\usepackage{caption}
\usepackage{booktabs}
\usepackage{makecell}

\newcommand{\change}[1]{\textcolor{black}{#1}}

\usepackage[T1]{fontenc}
\usepackage[font=small,labelfont=bf,tableposition=top]{caption}

\DeclareCaptionLabelFormat{andtable}{#1~#2  \&  \tablename~\thetable}

\newtheorem*{rep@theorem}{\rep@title}
\newcommand{\newreptheorem}[2]{%
\newenvironment{rep#1}[1]{%
 \def\rep@title{#2 \ref{##1}}%
 \begin{rep@theorem}}%
 {\end{rep@theorem}}}
\newtheorem{theorem}{Theorem}
\newtheorem{proposition}{Proposition}
\newtheorem{fact}{Fact}
\newtheorem{definition}{Definition}

\newtheorem{claim}{Claim}
\newreptheorem{theorem}{Theorem}
\newreptheorem{claim}{Claim}
\newcommand{\norm}[1]{\left\lVert#1\right\rVert}
\newcommand{\R}{\mathbb{R}}
\newcommand{\E}{\mathbb{E}}

\newcommand{\bs}[1]{\boldsymbol{#1}}
\newcommand{\bv}[1]{\mathbf{#1}}

\DeclareMathOperator{\argmax}{arg\,max}

\DeclareMathOperator*{\argmin}{arg\,min}

\newcommand{\x}{\mathbf{x}}
\newcommand{\w}{\mathbf{w}}
\newcommand{\g}{\mathbf{g}}
\renewcommand{\O}{\mathcal{O}}
\newcommand{\nnz}{\mathtt{nnz}}

\newcommand{\K}{\mathrm{K}} 
\newcommand{\M}{\mathrm{M}} 

\usepackage{multirow}
\usepackage{array}
\usepackage{url}

\usepackage[sc,osf]{mathpazo}

  \usepackage{nth}
  \usepackage{intcalc}

\title {Faster Kernel Interpolation for Gaussian Processes}

\author{Mohit Yadav, Daniel Sheldon, Cameron Musco \\
\normalsize{University of Massachusetts Amherst} \\ 
\normalsize{\texttt{\{ymohit, sheldon, cmusco\}@cs.umass.edu}}}
\date{}

\begin{document}
\maketitle
\begin{abstract}%
 A key challenge in scaling Gaussian Process (GP) regression to massive datasets is that exact inference requires computation with a dense $n \times n$ kernel matrix, where $n$ is the number of data points.
  Significant work focuses on approximating the kernel matrix via interpolation using a smaller set of $m$ ``inducing points''.
  \emph{Structured kernel interpolation} (SKI) is among the most scalable methods: by placing inducing points on a dense grid and using structured matrix algebra, SKI achieves per-iteration time of $\mathcal{O}(n + m \log m)$ for approximate inference. This linear scaling in $n$ enables inference for very large data sets; however the cost is \emph{per-iteration}, which remains a limitation for extremely large $n$.
We show that the SKI per-iteration time can be reduced to $\mathcal{O}(m \log m)$ after a single $\mathcal{O}(n)$ time precomputation step by reframing SKI as solving a natural Bayesian linear regression problem with a fixed set of $m$ compact basis functions. 
\change{With per-iteration complexity \emph{independent of the dataset size $n$} for a fixed grid, our method scales to truly massive data sets. We demonstrate speedups in practice for a wide range of $m$ and $n$ and apply the method to GP inference on a three-dimensional weather radar dataset with over 100 million points.} Our code is available at \url{https://github.com/ymohit/fkigp}.

\end{abstract}

\section{Introduction}
\label{sec:introduction}

GPs are a widely used and principled class of methods for predictive modeling.
They have a long history in spatial statistics and geostatistics for spatio-temporal interpolation problems~\cite{matheron1973intrinsic,cressie2008fixed}. They were later adopted in ML as general-purpose predictive models~\cite{rasmussen2004gaussian}, motivated in part by connections to neural networks~\cite{neal1995bayesian,williams1997computing}. More recently, similar connections have been identified between GPs and deep networks~\cite{lee2018deep,matthews2018gaussian,garriga2018deep,novak2018bayesian,cheng2019bayesian}.
GPs can be used for general-purpose Bayesian regression~\cite{williams1996gaussian,rasmussen1999evaluation}, classification~\cite{williams1998bayesian}, and many other applications~\cite{rasmussen2004gaussian,wilson2013gaussian}.

A well-known limitation of GPs is running-time scalability. The basic inference and learning tasks require linear algebraic operations (e.g., matrix inversion, linear solves, computing log-determinants) with an $n \times n$ kernel matrix, where $n$ is the number of data points. 
Exact computations require $\Theta(n^3)$ time --- e.g., using the Cholesky decomposition --- which severely limits applicability to large problems. Hence, a large amount of work has been devoted to improving scalability of GP inference and learning through approximation. 
Most of this work is based on the idea of forming an approximate kernel matrix that includes low-rank structure, e.g., through the use of {inducing points} or random features \cite{williams2001using,Snelson2005SparseGP,quinonero2005unifying,rahimi2008random,le2013fastfood}. 
With rank-$m$ structure, the running time of key tasks can be reduced to $\Theta(nm^2 + m^3)$ \cite{quinonero2005unifying}. 
However, this scaling with $n$ and $m$ continues to limit the size of input data sets that can be handled, the approximation accuracy, or both.

Structured kernel interpolation (SKI) is a promising approach to further improve the scalability of GP methods \change{on relatively low-dimensional data}~\cite{kissgp}. 
In SKI, $m$ inducing points are placed on a regular grid, which, when combined with a stationary kernel covariance function, imposes extra structure in the approximate kernel matrix. Kernel matrix operations on the grid (i.e., multiplying with a vector) require only $\O(m \log m)$ time, and interpolating from the grid requires $\O(n)$ time. By combining structured kernel operations with iterative methods for numerical linear algebra, the running time to solve core GP tasks becomes $\O(k(n + m \log m))$ where $k$ is the number of iterations, and is usually much less than $m$ or $n$. 
The modest per-iteration runtime of $\O(n + m \log m)$ allows the modeler to select a very large number of inducing points.

We show how to further improve the scalability of SKI with respect to $n$ and scale to truly massive data sets.
We first show that the SKI approximation corresponds to exact inference in a Bayesian regression problem with a fixed set of $m$ compact spatial basis functions. 
This lets us reduce the per-iteration runtime to $\O(m \log m)$ --- completely \emph{independent} of $n$ --- after a one-time preprocessing cost of $\O(n)$ to compute sufficient statistics of the regression problem.  However, naive application of these ideas introduces undesirable trade-offs: while the per-iteration cost is better, we must solve linear systems that are computationally less desirable than the original ones --- e.g., they are asymmetric instead of symmetric, or have worse condition number. To avoid these trade-offs, we contribute novel ``factorized" conjugate gradient and Lanczos methods, which allow us to solve the \emph{original} linear systems in $\O(m \log m)$ time per-iteration instead of $\O(n + m \log m)$. 

Our techniques accelerate SKI inference and learning across a wide range of settings. 
They apply to each of the main sub-tasks of GP inference: computing the posterior mean, posterior covariance, and log-likelihood. \change{We demonstrate runtime improvements across different data sets and grid sizes, and the ability to scale GP inference to datasets well outside the typical range that can be handled by SKI or other approaches, such as inducing point methods. For example, we demonstrate the ability to perform GP inference on a three-dimensional weather radar dataset with $120$ million data points, using a grid of $128,000$ inducing points.}

\subsection{Related Work}

\change{Outside of SKI and its variants \cite{wilson2015thoughts,gardner2018product}, a variety of scalable GP methods have been proposed. Most notable are 
 inducing point methods (sparse GP approximations), such as the Nystr\"{o}m, SoRs, FITC, and SMGP methods \cite{williams2001using,Snelson2005SparseGP,quinonero2005unifying,walder2008sparse}. These methods require either $\Omega (nm^2)$ time for direct solves, or $\Omega(nm)$ per-iteration cost if using iterative methods. 
 Our approach significantly improves the dependence on both $n$ and $m$ to just $\mathcal{O}(n)$ preprocessing time and $\O(m \log m)$ per-iteration cost.}

\change{While the  above methods generally do not leverage structured matrix methods like SKI, especially in higher dimensions, they may achieve comparable accuracy with a smaller number of inducing points. 
Directly comparing SKI to popular inducing point methods is beyond the scope of this work, however prior work shows significant performance gains on large, relatively low-dimensional datasets \cite{kissgp,ScalableLD}. We note that very recent work \cite{meanti2020kernel} seeks to push the limits of inducing point methods via careful systems and hardware level implementations. Like our work, they scale to datasets with over 100 million points.}

\change{Many scalable GP methods have also been proposed in the geostatistics literature -- see \cite{heaton2019case} for a survey. These include structured methods when observations lie on a grid  \cite{guinness2017circulant,stroud2017bayesian}. Most closely related to our work is \emph{fixed rank kriging} (FRK) \cite{cressie2008fixed}, which can be viewed as a generalization of our Bayesian regression interpretation of SKI. Like inducing point methods, FRK using $m$ basis functions requires $\Omega (nm^2)$ time.
We show that SKI can be viewed as a special case of FRK with a fixed kernel function and  a particular set of basis functions arising from interpolation. These two choices allow our faster runtime, through the application of structured matrix techniques and factorized iterative methods, which in turn significantly increases the number of basis functions that can be used.}

\section{Background}
\label{sec:background}

\noindent \textbf{Notation:} Throughout we use bold letters to represent vectors and capitals to represent matrices. $I \in \R^{n \times n}$ represents the identity matrix, with dimension apparent from context. For $M \in \R^{p \times r}$, $mv(M)$ denotes the time required to multiply $M$ by any vector in $\R^r$.

\subsection{Gaussian Process Regression}
\label{sec:gp}

In GP regression~\cite{rasmussen2004gaussian}, response values are modeled as noisy measurements of a random function $f$ on input data points.
Let $\mathcal{D}$ be a set of $n$ points $\bv x_1,\ldots,\bv x_n \in \R^d$ with corresponding response values $y_1,\ldots,y_n \in \R$.
Let $\bv{y} \in \R^{n}$ have its $i^{th}$ entry equal to $y_i$ and $X \in \R^{n \times d}$ have its $i^{th}$ row equal to $\bv x_i$.
A Gaussian process with kernel (covariance) function $k(\mathbf{x}, \mathbf{x}^{\prime}): \R^d \times \R^d \rightarrow \R$ is a random function $f$ such, for any $\bv x_1,\ldots, \bv x_n \in \R^d$:
\begin{align}
\label{eq:gp_prior}
\bv{f} = [f(\mathbf{x}_1),...,f(\mathbf{x}_n)] \sim \mathcal{N}(0, K_{X}),
\end{align}
where $K_{X} = [k(\mathbf{x}_i,\mathbf{x}_j)]_{i,j=1}^{n} \in \mathbb{R}^{n \times n}$ is the kernel (covariance) matrix on the data points $X$.
We assume without loss of generality that $f$ is zero-mean.
The responses $\bv y$ are modeled as measurements of $\bv f_X$ with i.i.d. Gaussian noise, i.e.,
$\bv{y} \sim \mathcal{N}(\bv{f}, \sigma^2 I)$.

The posterior  distribution of $f$ given the data $\mathcal{D} = (X, \mathbf{y})$ is itself a Gaussian process. The standard GP inference tasks are to compute the posterior mean and covariance and the log-likelihood, given in Fact~\ref{fact:gp_inference}.

\begin{mdframed}[backgroundcolor=light-gray] 
\begin{fact}[\textbf{Exact GP Inference}]\label{fact:gp_inference} The posterior mean, covariance, and log likelihood for Gaussian process regression are given by:
\begin{align*}
& \textbf{mean: } \mu_{f|\mathcal{D}}(\mathbf{x}) = \bv{k}_\bv{x}^T \bv{z}\\ 
& \textbf{covariance: }  k_{f|\mathcal{D}}(\mathbf{x}, \mathbf{x}^{\prime}) = k(\mathbf{x}, \mathbf{x}^{\prime})
- \bv{k}_{\bv{x}}^T ({K}_{X}+\sigma^2 I)^{-1}  \bv{k}_{\mathbf{x}^{\prime}}\nonumber\\
& \textbf{log likelihood: } \log \Pr(\bv{y}) = -\frac{1}{2} [\log\det ({K}_{X}+\sigma^2 I)+ \bv{y}^T \bv{z}+ n \log (2\pi) ]\nonumber
\end{align*}
where $\bv{k}_{\bv x} \in \R^{n}$ has $i^{th}$ entry $ k(\bv{x},\bv{x}_i)$ and $\bv{z} = ({K}_{X}+\sigma^2 I)^{-1} \mathbf{y}$.
\end{fact}
\end{mdframed}

Evaluating the posterior mean $\mu_{f|\mathcal{D}}(\mathbf{x})$, covariance $k_{f|\mathcal{D}}(\mathbf{x}, \mathbf{x}^{\prime})$, and log likelihood require matrix-vector multiplication (MVM) with  $({K}_{X} + \sigma^2 I)^{-1}$, which is the major run-time bottleneck for GP inference. 
Computing this inverse directly requires $\Theta(n^3)$ time. The log likelihood requires a further computation of $\log\det ({K}_{X}+\sigma^2 I)$, which again naively takes $\Theta(n^3)$ time using the Cholesky decomposition. Computing its gradient, which is necessary e.g., in hyper-parameter tuning with gradient based methods, requires a further trace computation involving $({K}_{X} + \sigma^2 I)^{-1}$.

\paragraph{Inference Via Iterative Methods.}\label{sec:iterative}
One way to accelerate GP inference is to avoid full matrix factorization like Cholesky decomposition and instead 
use iterative methods.
Gardner et al. \cite{GPyTorchBM} detail how to compute or approximate each term in Fact~\ref{fact:gp_inference} using a modified version of the conjugate gradient (CG) algorithm.

For example, the vector $\bv{z} = ({K}_{X} + \sigma^2 I)^{-1} \bv{y} \in \R^n$ is computed by using CG to solve $({K}_{X} + \sigma^2 I)\bv{z} = \bv{y}$, which yields the posterior mean as $\mu_{f|\mathcal{D}}(\mathbf{x}) = \bv{k}_\bv{x}^T \bv{z}$, and is also used in the calculation of the log-likelihood. 

We will use two iterative algorithms in our work: (1) conjugate gradient to solve linear systems $A\bv{v} = \bv{b}$ for symmetric positive definite (SPD) $A$, (2) the Lanczos algorithm to (sometimes partially) \emph{tridiagonalize} an SPD matrix $A \in \R^{p \times p}$ as $A=QTQ^T$ where $Q \in \R^{p \times p}$ is orthonormal and $T \in \R^{p \times p}$ is tridiagonal.
Tridiagonalization is used to approximate $\log\det(\cdot)$ terms~\cite{ScalableLD,ubaru2017fast} and to compute a low-rank factorization for approximate posterior covariance evaluation~\cite{ConstantTimePD}.

Each iteration of CG or Lanczos for GP inference requires matrix-vector multiplication with $A = {K}_{X} + \sigma^2 I$, which in general takes $mv(K_X) + n = \O(n^2)$ time, but may be faster if $K_X$ has special structure. The number of iterations required to reach a given error tolerance depends on eigenspectrum of $A$, and is usually much less than $n$, often on the order of 50 to 100. It can be even lower with preconditioning~\cite{GPyTorchBM}. Recent work thus often informally considers the number of iterations to be a constant independent of $n$.
 
\subsection{Structured Kernel Interpolation}
\label{sec:ski}
For large $n$, the $\Theta(n^2)$ per-iteration cost of iterative methods is still prohibitively expensive. 
Structured kernel interpolation (SKI) is a method to accelerate MVMs by using an approximate kernel matrix with a special structure~\cite{kissgp}.
SKI approximates the kernel $k(\bv{x},\bv{x}')$ as:
\begin{equation}
  \tilde k(\bv{x},\bv{x'}) = \bv{w}_{\bv{x}}^T K_{G} \bv{w}_{\bv{x}'}
  \label{eq:ski}
\end{equation}
where $K_{G} \in \R^{m \times m}$ is the kernel matrix for the set of $m$ points on a dense $d$-dimensional grid, and
the vector $\bv{w}_\bv{x} \in \R^m$ contains interpolation weights to interpolate from grid points to arbitrary $\bv x \in \R^d$.
I.e., the kernel inner product between any two points 
is approximated by interpolating kernel inner products among grid points.

SKI can use any interpolation strategy (e.g., linear or cubic); typically, the strategy is local, so that $\bv{w}_\bv{x}$ has only a constant number of non-zero entries corresponding to the grid points closest to $\bv x$.
E.g., for linear interpolation, $\bv{w}_x$ has $2^d$ non-zeros.
Let $W \in \R^{n \times m}$ have $i^{th}$ row equal to $\bv{w}_{\bv{x}_i}$. SKI
approximates the true kernel matrix $K_{X}$ as $\tilde K_{X} = W K_{G} W^T$.
Plugging this approximation directly into the GP inference equations of Fact \ref{fact:gp_inference} yields the SKI inference scheme  in Def.~\ref{def:kissgp}.

\begin{mdframed}[backgroundcolor=light-gray] 
\begin{definition}[\textbf{SKI Inference}]\label{def:kissgp}
The SKI approximate inference equations are given by:
\begin{align*}
 & \textbf{mean: } \mu_{f|\mathcal{D}}(\mathbf{x}) \approx \bv{w}_{\bv{x}}^T K_G W^T \bv{\tilde z} \\
& \textbf{covariance: } k_{f|\mathcal{D}}(\mathbf{x}, \mathbf{x}^{\prime}) \approx \bv{w}_\bv{x}^T K_G \bv{w}_{\bv x'} - \bv{\tilde k}_\bv{x}^{T} (\tilde K_{X}+ \sigma^2 I)^{-1} \bv{\tilde k}_\bv{x} \nonumber \\
& \textbf{log likelihood: } \log \Pr(\bv{y}) \approx -\frac{1}{2} [\log\det ( \tilde K_{X}+ \sigma^2 I) + \bv{y}^T \bv{\tilde z} + n \log (2\pi) ]\nonumber
\end{align*}
where $\tilde K_X = W K_{G}W^{T}$ and $\bv{\tilde z} = \left (\tilde K_X + \sigma^2 I\right )^{-1} \mathbf{y}$ and $\bv{\tilde k}_\bv{x} = W K_{G} \bv{w}_{\bv{x}}$.
\end{definition}
\end{mdframed}

 \paragraph{SKI Running Time and Memory.}\label{sec:skiRuntime}
 
The SKI method admits efficient approximate inference due to: (1) the sparsity of the interpolation weight matrix $W$, and (2) the structure of the on-grid kernel matrix $K_G$.
The cost per iteration of CG or Lanczos is $\mathcal{O}(mv(\tilde K_X + \sigma^2 I)) = \mathcal{O}(mv(W) + mv(K_G) + n)$.
This runtime is $\mathcal{O}(n + m \log m)$ per iteration assuming: (1) $W$ has $\mathcal{O}(1)$ entries per row and so $mv(W) = \mathcal{O}(n)$, and (2) $K_G$ is multilevel Toeplitz, so $mv(K_G) = \mathcal{O}(m \log m)$ via fast Fourier transform \cite{lee1986fast}. The matrix $K_G$ is multilevel Toeplitz (also known as block Toeplitz with Toeplitz blocks or BTTB)  whenever $G$ is an equally-spaced grid and $k(\cdot, \cdot)$ is stationary.
The memory footprint is roughly $\nnz(W) + m + n = \O(n + m)$ to store $W$, $K_G$, and $\bv y$, respectively,  where $\nnz(A)$ denotes the number of non zeros of $A$.

Overall, the SKI per-iteration runtime of $\mathcal{O}(n + m \log m)$ significantly improves on the naive $\mathcal O(n^2)$ time required to apply the true kernel matrix $K_{X}$. However, when $n$ is very large, the $\mathcal{O}(n)$ term (for both runtime and memory) can become a bottleneck. Our main contribution is to remove this cost, giving methods with $\mathcal{O}(m \log m)$ per-iteration runtime with $\O(m)$ memory after $\mathcal{O}(n)$ preprocessing. 
\section{SKI as Bayesian Linear Regression with Fixed Basis Functions}
\label{sec:gsgp}

Our first contribution is to reformulate SKI as \emph{exact inference} in a Bayesian linear regression problem with compact basis functions associated with grid points. This lets us use standard conjugate update formulas for Bayesian linear regression to reduce SKI's per-iteration runtime to $\mathcal O(m \log m)$, with $\mathcal O(n)$ preprocessing.

\begin{wrapfigure}{R}{0.42\textwidth}
  \includegraphics[width=0.4\textwidth]{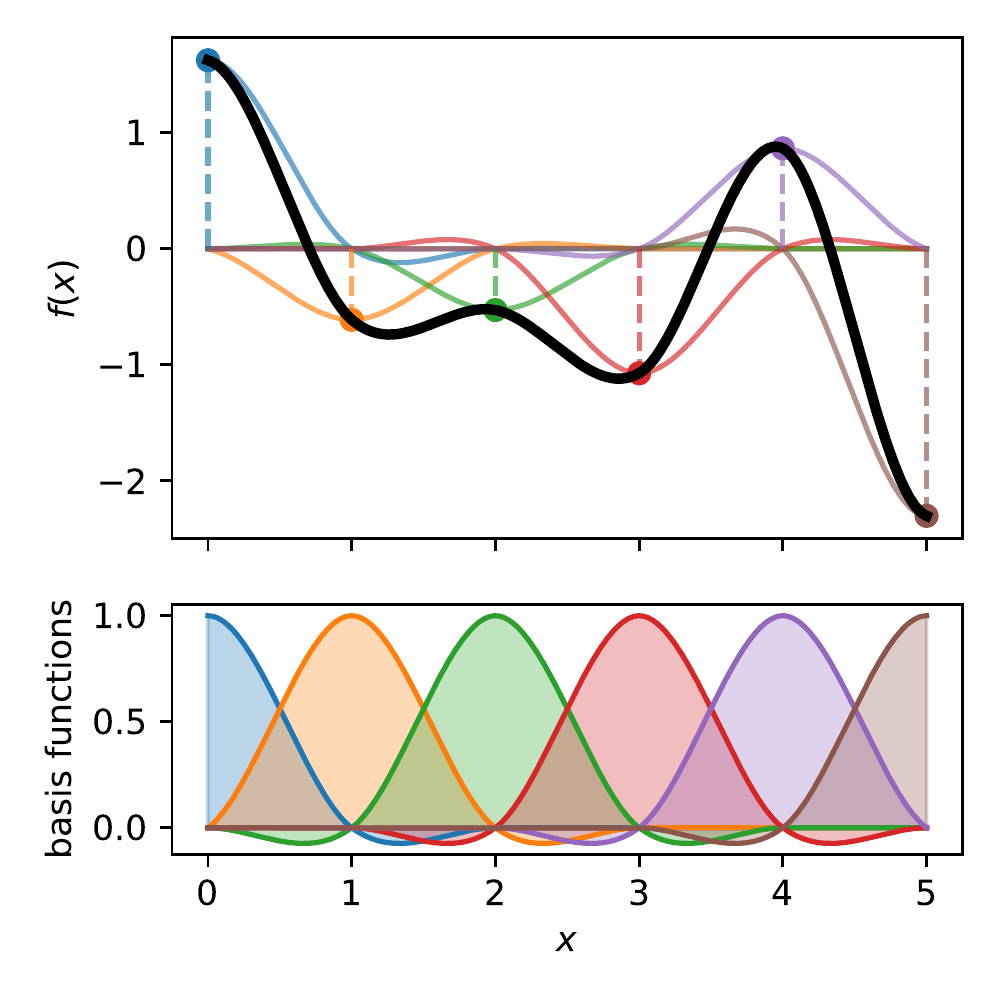}
  \caption{GSGP illustration. Bottom: basis functions for cubic interpolation are compact and centered at grid points. Top: a GSGP (thick black curve) is formed as the sum of scaled basis functions (lighter colored curves) with random weights at grid points  (vertical dashed lines) drawn from the original GP. \label{fig:gsgp}}
\end{wrapfigure} 

\begin{mdframed}[backgroundcolor=light-gray] 
\begin{definition}[\textbf{Grid-Structured Gaussian Process; GSGP}]\label{def:gsgp}
  Let $G = \{\bv g_1,\ldots, \bv g_m\} \subseteq \R^d$ be a set of grid points and $k(\mathbf{x}, \mathbf{x}^{\prime}): \R^d \times \R^d \rightarrow \R$ be a positive-definite kernel function. A \emph{grid-structured Gaussian process} $f$ is defined by the following generative process:
  \begin{align*}
    \boldsymbol{\theta} &\sim \mathcal{N}(0, K_{G}), \\
    f(\bv x) &= \bv{w}_{\bv{x}}^T \boldsymbol{\theta}, \quad \forall \bv x \in \R^d.
  \end{align*}
 where $\bv{w}_x \in \R^m$ is a vector of interpolation weights from $\bv{x}$ to the grid $G$. 
\end{definition}
\end{mdframed}

Notice that a GSGP is a classical Bayesian linear regression model. In principle, $\bv{w}$ can be any mapping from $\R^d$ to $\R^m$. However, for computational efficiency and to match the notion of interpolation, the vector $\bv{w}_{\bv{x}}$ will be taken to be the set of weights used by any fixed scheme to interpolate values from grid points to arbitrary locations $\bv x \in \R^d$.

The generative process is illustrated in Figure~\ref{fig:gsgp} for $d=1$ and cubic interpolation on an integer grid~\cite{keys1981cubic}. The basis functions $\bv{w}^j_\bv{x}$ for each grid point $j$ and for all $\bv x$ are shown in the bottom panel. The $j$th basis function is centered at $j$ and supported on $[j -2, j+2)$. For any fixed $\bv x$, the vector $\bv{w}_\bv{x}$ has at most four nonzero entries corresponding to the four nearest grid points.

The GSGP $f$ is generated by first drawing random weights $\boldsymbol{\theta}$ as the values of the \emph{original} GP --- with covariance function $k(\cdot, \cdot)$ --- at grid points. The generated function $f(\bv x)$ can be interpreted in two ways: (1) a sum of scaled basis functions $f(\bv x) = \sum_j \bs{\theta}_j \bv w^j_\bv{x}$, (2) the result of interpolating the grid values $\boldsymbol{\theta}$ to $\R^d$ using the interpolation scheme.

It is straightforward to verify the following (full derivations and proofs appear in Appendix~\ref{app:gsgp}):

\begin{claim}A GSGP with grid weights $\boldsymbol{\theta}$ drawn from a GP with covariance function $k(\bv{x}, \bv{x}')$ is itself a GP with covariance function $\tilde{k}(\bv x, \bv x') = \bv{w}_{\bv{x}}^T K_G \bv{w}_{\bv{x}'}$.
\label{claim:gsgpski}
\end{claim}
In other words: the \emph{exact} covariance function of the GSGP is the same as the SKI approximation in Eq.~\eqref{eq:ski}.
Now, suppose noisy observations $\bv y \sim \mathcal{N}(\bv f_X, \sigma^2)$ are made of the GSGP at input locations $X$. It is well known that the posterior distribution of $f$ is a Gaussian process, with mean, covariance and log likelihood given in Fact~\ref{fact:gsgp}~\cite{bayesianRegression,bishop2006pattern}. From Claim~\ref{claim:gsgpski} it follows that:
\begin{theorem}[Equivalence of GSGP Inference and SKI Approximation]\label{thm:gsgp}
The inference expressions of Fact \ref{fact:gsgp} are identical to the SKI approximations of Def. \ref{def:kissgp}.
\end{theorem}

\begin{figure}[h]
\begin{mdframed}[backgroundcolor=light-gray] 
\begin{fact}[\textbf{GSGP Inference}]\label{fact:gsgp} The posterior mean, covariance, and log likelihood functions for the grid-structured Gaussian process (Def. \ref{def:gsgp}) are given by:
\begin{align*}
&\textbf{mean: } \mu_{f|\mathcal{D}}(\mathbf{x}) = \bv w_{\bv{x}}^T  \bv{\bar z} \\
&\textbf{covariance: } k_{f|\mathcal{D}}(\mathbf{x}, \mathbf{x}^{\prime}) =  \bv{w}_\bv{x}^T \bv{\bar{C}} \bv{w}_\bv{x'}\nonumber\\
&\textbf{log likelihood: } \log \Pr(\bv{y}) =  -\frac{1}{2} \big [  \log \det(K_{G} W^T W + \sigma^2 I) + \frac{\bv{y}^T (\bv{y} - W\bv{\bar z})}{\sigma^2} + c\big ]\nonumber, 
\end{align*}
where $\bv{\bar z} = \E[\boldsymbol{\theta} | \bv y] = (K_{G} W^T W + \sigma^2 I)^{-1} K_{G} W^{T} \bv y$ is the posterior mean of $\boldsymbol{\theta}$, ${\bv{\bar C}} = \var(\boldsymbol{\theta} | \bv y) = \sigma^2 \left( K_G W^T W   + \sigma^2 I \right)^{-1}  K_{G}$ is the posterior variance of $\boldsymbol{\theta}$ and $c = n \log (2\pi) +(n-m)\log \sigma^2$.
\end{fact}
\end{mdframed}
\end{figure}

\subsection{GSGP Running Time and Memory}\label{sec:gsgpIter}
 
By Theorem \ref{thm:gsgp}, we can apply the SKI approximation using the GSGP inference equations of Fact \ref{fact:gsgp};
these also involve structured matrices that are well suited to iterative methods.
In particular, they require linear solves and logdet computation for the $m \times m$ matrix $K_GW^T W + \sigma^2 I$ rather than the $n \times n$ matrix $W K_G W^T + \sigma^2 I$. Under the standard SKI assumptions, this leads to $\mathcal{O}(m \log m)$ per-iteration run time and $\O(m)$ memory footprint with $\mathcal{O}(n)$ precomputation. 
 
\textbf{Precomputation:} GSGP involves precomputing $W^T W \in \R^{m \times m}$, $W^T \bv y \in \R^m$, and $\bv y^T \bv y \in \R$, which are the sufficient statistics of a linear regression problem with feature matrix $W$. Each is a sum over $n$ data points and has fixed size depending only on $m$. Once computed, each expression in Fact~\ref{fact:gsgp} can be computed without referring back to the data $W$ and $\bv y$. 

It is clear that $\bv y^T \bv y = \sum_{i=1}^n y_i^2$ can be computed in $\mathcal{O}(n)$ time with a single pass over the data.
Assume the interpolation strategy is local (e.g., linear or cubic interpolation), so that each $\bv{w}_{\bv{x}_i}$ has $\mathcal{O}(1)$ non-zeros.
Then $W^T \bv y = \sum_{i=1}^n \bv{w}_{\bv{x}_i}^T \bv y$ can also be computed in $\mathcal{O}(n)$ time with one pass over the data, since each inner product accesses only a constant number of entries of $\bv y$. $W^T W$ also has desirable computational properties:
\begin{claim}
  Assume that $G = \{\bv{g}_1,\ldots,\bv{g}_m\}$ has spacing $s$, i.e., $\norm{\bv{g}_i-\bv{g}_j}_\infty \ge s$ for any $i,j \in m$ and that $\w^j_{\x}$ is non-zero only if $\|\g_j - \x\|_\infty < r \cdot s$ for some fixed integer $r$. Then $W^T W$ can be computed in $\mathcal{O}(n(2r)^{2d})$ time and has at most $(4r-1)^{d}$ entries per row. Therefore $mv(W^TW) = \mathcal{O}(m (4r-1)^d).$
  \label{claim:WTW}
\end{claim}

For example, $r=1$ for linear interpolation and $r=2$ for cubic interpolation. The upshot is that $W^TW$ can be precomputed in $\mathcal{O}(n)$ time, after which matrix-vector multiplications take $\mathcal{O}(m)$ time, with dependence on $r$ and $d$ similar to that of $mv(W) = \mathcal{O}(n(2r)^d)$.


 \textbf{Per-Iteration and Memory:}
 As discussed in Section \ref{sec:skiRuntime}, due to its grid structure, $K_{G}$ admits fast matrix-vector multiplication: $mv(K_G) = \mathcal O(m\log m)$ for stationary kernels. Since $W^T W$ is sparse, $mv(W^T W) = \mathcal O(m)$. Overall, $mv(K_{G} W^T W + \sigma^2 I) = \mathcal O(m \log m)$, giving per-iteration runtime of $\O(m \log m)$
 for computing the approximate mean, covariance, and log-likelihood in Fact \ref{fact:gsgp} via iterative methods. Importantly, this complexity is \emph{independent} of the number of data points $n$. GSGP uses $\nnz(W^TW) + m + m = \O(m)$ memory to store $W^TW$, $W^T \bv y$ and $K_G$.
 
\paragraph{Limitations of GSGP.}
Directly replacing the classic SKI method with the GSGP inference equations of Fact \ref{fact:gsgp} reduces per-iteration cost but has some undesirable trade-offs. In particular, unlike $W K_G W^T + \sigma^2 I$, the matrix $K_G W^T W + \sigma^2 I$ is \emph{asymmetric}. Thus, conjugate gradient and Lanczos---which are designed for \emph{symmetric} positive semidefinite matrices---are not applicable. Asymmetric solvers like GMRES~\cite{saad1986gmres} can be used, and seem to work well in practice for posterior mean estimation, but do not enjoy the same theoretical convergence guarantees, nor do they as readily provide the approximate tridiagonalization for low-rank approximation for predictive covariance~\cite{ConstantTimePD} or log-likelihood estimation~\cite{ScalableLD}. It is possible to algebraically manipulate the GSGP expressions to yield \emph{symmetric} $m \times m$ systems, but these lose the desirable `regularized form' $A + \sigma^2 I$ and have worse conditioning, leading to more iterations being required in practice (see Appendix \ref{app:gsgp}).




\section{Efficient SKI via Factorized Iterative Methods}
\label{sec:factorized}

In this section we show how to achieve the best of both SKI and 
the GSGP reformulation: we design `factorized' versions of the CG and Lanczos methods used in SKI with just $\mathcal{O}(m \log m)$ per-iteration complexity. These methods are mathematically equivalent to the full methods, and so enjoy identical convergence rates, avoiding the complications of the asymmetric solves required by GSGP inference.


\subsection{The Factorized Approach}

Our approach centers on a simple observation relating matrix-vector multiplication with the SKI kernel approximation $\tilde K_X = W K_G W^T + \sigma^{2} I $ and the GSGP operator $K_G W^T W + \sigma^2 I$. To apply the SKI equations of Def. \ref{def:kissgp} via an iterative method, a key step at each iteration is to multiply some iterate $\bv{z}_i \in \R^n$ by $\tilde K_X$, requiring $\mathcal O(n + m \log m)$ time. We avoid this by maintaining a compressed representation of any iterate as $\bv{z}_i = W \bv{\hat z}_i + c_i \bv{z}_0$, where $\bv{\hat z}_i \in \R^m$, $c_i \in \R$ is a scalar coefficient, and  $\bv{z}_0$ is an initial value. At initialization, $\bv{\hat z}_0 = \bv{0}$ and $c_0 = 1$. Critically, this compressed representation can be updated with multiplication by $\tilde K_X$ in just $\mathcal O(m\log m)$ time using the following claim:
\begin{claim}[Factorized Matrix-Vector Multiplication]\label{clm:goldenRule}
For any $\bv{z}_i \in \R^n$ with $\bv{z}_i = W \bv{\hat z}_i + c_i \bv{z}_0$, 
\begin{align*}
(WK_GW^T + \sigma^2 I) \bv{z}_i = W \bv{\hat z}_{i+1} + c_{i+1} \bv{z}_0,
\end{align*}
where $\bv{\hat z}_{i+1} = (K_G W^T W + \sigma^2 I) \bv{\hat z}_{i} + c_i K_G W^T \bv{z}_0$ and $c_{i+1} = \sigma^2 \cdot c_i$.
Call this operation a \emph{factorized update} and denote it as $(\bv{\hat z}_{i+1}, c_{i+1}) = \mathcal{A}(\bv{\hat z}_{i},c_i)$.
If the vector $K_G W^T \bv{z}_0$ is precomputed in $\mathcal{O}(n + m \log m)$ time, each subsequent factorized update takes $\mathcal O(m \log m)$ time.
\end{claim}
For algorithms such as CG we also need to support additions and inner products in the compressed representation. Additions are simple via linearity. Inner products can be computed efficiently as well:
\begin{claim}[Factorized Inner Products]\label{clm:goldenRule2}
For any $\bv{z}_i,\bv{y}_i \in \R^n$ with $\bv{z}_i = W \bv{\hat z}_i + c_{i} \bv{z}_0$ and $\bv{y}_i = W \bv{\hat y}_i + d_{i} \bv{y}_0$,
\begin{align*}
\bv{z}_i^T \bv{y}_i &= \bv{\hat z}_i^T W^T W \bv{\hat y}_i + d_{i} \bv{\hat z}_i^T W^T  \bv{y}_0 \\
							 &+ c_{i} \bv{\hat y}_i^T W^T  \bv{z}_0 +   c_{i}  d_i  \bv{y}_0^T  \bv{z}_0.
\end{align*}
We denote the above operation by $\langle (\bv{\hat z}_i, c_{i}), (\bv{\hat y}_i, d_{i}) \rangle$. If ${W}^T \bv{z}_0$, ${W}^T \bv{y}_0$, and $\bv{y}_0^T \bv{z}_0$ are precomputed in $\mathcal{O}(n)$ time, then $\langle (\bv{\hat z}_i, c_{i}), (\bv{\hat y}_i, d_{i}) \rangle$ can be computed using just one matrix-vector multiplication with $W^T W$ and $\mathcal{O}(m)$ additional time.
\end{claim}

\subsection{Factorized Conjugate Gradient}

We now give an example of this approach by deriving a ``factorized conjugate gradient'' algorithm. Factorized CG has lower per-iteration complexity than standard CG for computing the posterior mean, covariance, and log likelihood in the SKI approximations of Def. \ref{def:kissgp}. In Appendix \ref{app:factorized} we apply the same approach to the Lanczos method, which can be  used in approximating the logdet term in the log likelihood, and for computing a low-rank approximation of the posterior covariance.

Figure \ref{fig:fcg} shows a side by side of the classic CG method and our factorized variant. 
CG maintains three iterates, the current solution estimate $\bv{x}_k$, the residual $\bv{r}_k$, and a search direction $\bv{p}_k$. We maintain each in a compressed form with $\bv{x}_k = W \bv{\hat x}_k + c_{k}^x \bv{r}_0 + \bv{x}_0$, $\bv{r}_k = W \bv{\hat r}_k + c_{k}^r \bv{r}_0$, and $\bv{p}_k = W \bv{\hat p}_k + c_{k}^p \bv{r}_0$. Note that the initialization term $\bv{r}_0$ is shared across all iterates with a different coefficient, and that $\bv{x}_k$ has an additional fixed component $\bv{x}_0$. This is an initial solution guess for the system solve, frequently zero.
With these invariants, simply applying Claims~\ref{clm:goldenRule} and \ref{clm:goldenRule2} gives the factorized algorithm.

\begin{figure}[h]
\begin{minipage}{0.38\textwidth}
\begin{algorithm}[H]
\caption{Conjugate gradient}
\label{alg:cg}
\begin{algorithmic}[1]
\Procedure {CG}{$K_G, W, \bv{b}, \sigma, \mathbf{x}_0, \epsilon$}
\State $\mathbf{r}_{0} = \mathbf{b} -\tilde K \mathbf{x}_{0}$
\State $\mathbf {p}_{0} = \mathbf{r}_{0}$ 
\For{$k = 0$ to $\text{maxiter}$} 
\State $\alpha_{k} = {\frac{\mathbf{r}_{k}^T \mathbf{r}_{k}}{\mathbf {p} _{k}^{T} \tilde K \mathbf{p}_{k}}}$  
\State $\mathbf{x}_{k+1}  = \mathbf{x}_{k} + \alpha_{k} \cdot \mathbf {p}_{k} $ 
\State $\mathbf {r}_{k+1} = \mathbf {r}_{k} - \alpha _{k} \cdot \tilde K \mathbf {p}_{k} $ 
\State  if $\mathbf {r}_{k+1}^{T}\mathbf {r}_{k+1} \leq \epsilon $ exit loop 
\State $\beta_{k} = {\frac{\mathbf {r}_{k+1}^T\mathbf {r}_{k+1}}{\mathbf {r}_{k}^T\mathbf {r}_{k}}}$ 
\State $\mathbf{p}_{k+1} = \mathbf{r}_{k+1} + \beta_{k} \mathbf{p}_{k} $
\EndFor
\State \textbf{return} $\bv{x}_{k+1}$
\EndProcedure
\end{algorithmic}
  \end{algorithm}
\end{minipage}
\hfill
\begin{minipage}{0.58\textwidth}
\begin{algorithm}[H]
\caption{Factorized conjugate gradient (FCG)}
\label{alg:fcg}
\begin{algorithmic}[1]
\Procedure {FCG}{$K_G, W, \bv{b}, \sigma, \mathbf{x}_0, \epsilon$} 
\State $\bv{r_0} = \bv{b} - \tilde K \bv{x}_0,\, \bv{\hat r_0} = \bv{0},\, c_{0}^r = 1$
\State $\mathbf{\hat{p}}_{0} = \bv{0},\, c_{0}^p = 1$, $\, \bv{\hat x}_0 = \bv{0},\, c_{0}^x = 0$
\For{$k = 0$ to $\text{maxiter}$} 
\State $\alpha_{k} = \frac{\langle (\bv{\hat r}_k,c_{k}^r), (\bv{\hat r}_k,c_{k}^r)\rangle}{\langle (\bv{\hat{p}}_k,c_{k}^p), \mathcal{A}(\bv{\hat{p}}_k,c_{k}^p)\rangle } $ 
\State $(\mathbf{\hat x}_{k+1}, c_{k+1}^x) = (\mathbf{\hat x}_{k}, c_{k}^x)+  \alpha_{k} \cdot ( \mathbf{\hat{p}}_k, c_{k}^p)$ 
\State $(\mathbf{\hat{r}}_{k+1},c_{k+1}^r) = (\mathbf{\hat{r}}_{k},c_{k}^r)-  \alpha _{k} \cdot \mathcal{A}( \mathbf{\hat{p}}_k , c_{k}^p)$ 
\State  if $\langle (\bv{\hat r}_{k+1},c_{k+1}^r), (\bv{\hat r}_{k+1},c_{k+1}^r)\rangle \leq \epsilon $ exit loop 
\State $\beta_{k}  =  \frac{\langle (\bv{\hat r}_{k+1},c_{k+1}^r), (\bv{\hat r}_{k+1},c_{k+1}^r)\rangle}{\langle (\bv{\hat r}_k,c_{k}^r), (\bv{\hat r}_k,c_{k}^r)\rangle} $ 
\State $(\mathbf{\hat{p}}_{k+1},c_{k+1}^p) = (\mathbf{\hat{r}}_{k+1},c_{k+1}^r) + \beta_{k} \cdot (\mathbf{\hat{p}}_{k},c_{k}^p) $
\EndFor
\State \textbf{return} $\bv{x}_{k+1} = W \bv{ \hat x}_{k+1} + c_{k+1}^x \cdot \bv{r}_0 + \bv{x}_0$
\EndProcedure
\end{algorithmic}
\end{algorithm}
\end{minipage}
\caption{Above $\tilde K = W K_G W^T + \sigma^2 I$. $\mathcal{A}(\cdot)$ and $\langle \cdot,\cdot \rangle$ denote the factorized matrix-vector multiplication and inner product updates of Claims \ref{clm:goldenRule} and \ref{clm:goldenRule2}. The vector $\bv{x}_0 \in \R^n$ is an initial solution, the scalar $\epsilon > 0$ is a tolerance parameter, and $\text{maxiter}$ is the maximum number of iterations. }
\label{fig:fcg}
\end{figure}
\begin{proposition}[Factorized CG Equivalence and Runtime]\label{prop:fcg} 
The outputs of Algs. \ref{alg:cg} and \ref{alg:fcg} on the same inputs are identical. 
Alg. \ref{alg:fcg} performs two matrix-vector multiplications with $K_G$ and three with $W$ initially. In each iteration, it performs a constant number of multiplications with $K_G$ and $W^TW$ plus $\mathcal{O}(m)$ additional work. 
If $W^T W$ is sparse and $K_G$ has multilevel Toeplitz structure, its per iteration runtime is $\mathcal{O}(m \log m)$.
\end{proposition} 
Appendix \ref{app:factorized} presents a further optimization to only require one matrix-vector multiplication with $K_G$ and one with $W^T W$ per iteration.
A similar optimization applies to the factorized Lanczos method.

\section{Experiments}
\label{sec:experiments}

We conduct experiments to evaluate our ``GSGP approach'' of using factorized algorithms to solve SKI inference tasks. We use: (1)~factorized CG to solve the linear systems for the SKI posterior mean and covariance expressions of Def.~\ref{def:kissgp} and for approximate tridiagonalization within stochastic log-likelihood approximations~\cite{GPyTorchBM}, and (2)~factorized Lanczos for {low-rank} approximations of the SKI predictive covariance~\cite{ConstantTimePD}.

Our goals are to: (1)~evaluate the running-time improvements of GSGP, (2)~examine new speed-accuracy tradeoffs enabled by GSGP, and (3)~demonstrate the ability to push the frontier of GP inference for massive data sets. 
We use a synthetic data set and three real data sets from prior work~\cite{ScalableLD, kissgp, angell2018inferring}, summarized in Table~\ref{tab:per_iteration_time_memory_requirements}. \change{We focus on large, relatively low-dimensional datasets -- the regime targeted by structured kernel interpolation methods. The Radar dataset is a subset of a larger 120M point dataset. While SKI cannot scale to this data size without significant computational resources, at the end of the section we demonstrate that GSGP's ability to scale to this regime with modest runtime and memory usage.}

\begin{figure}[h]
    \centering
\includegraphics[width=8.5cm]{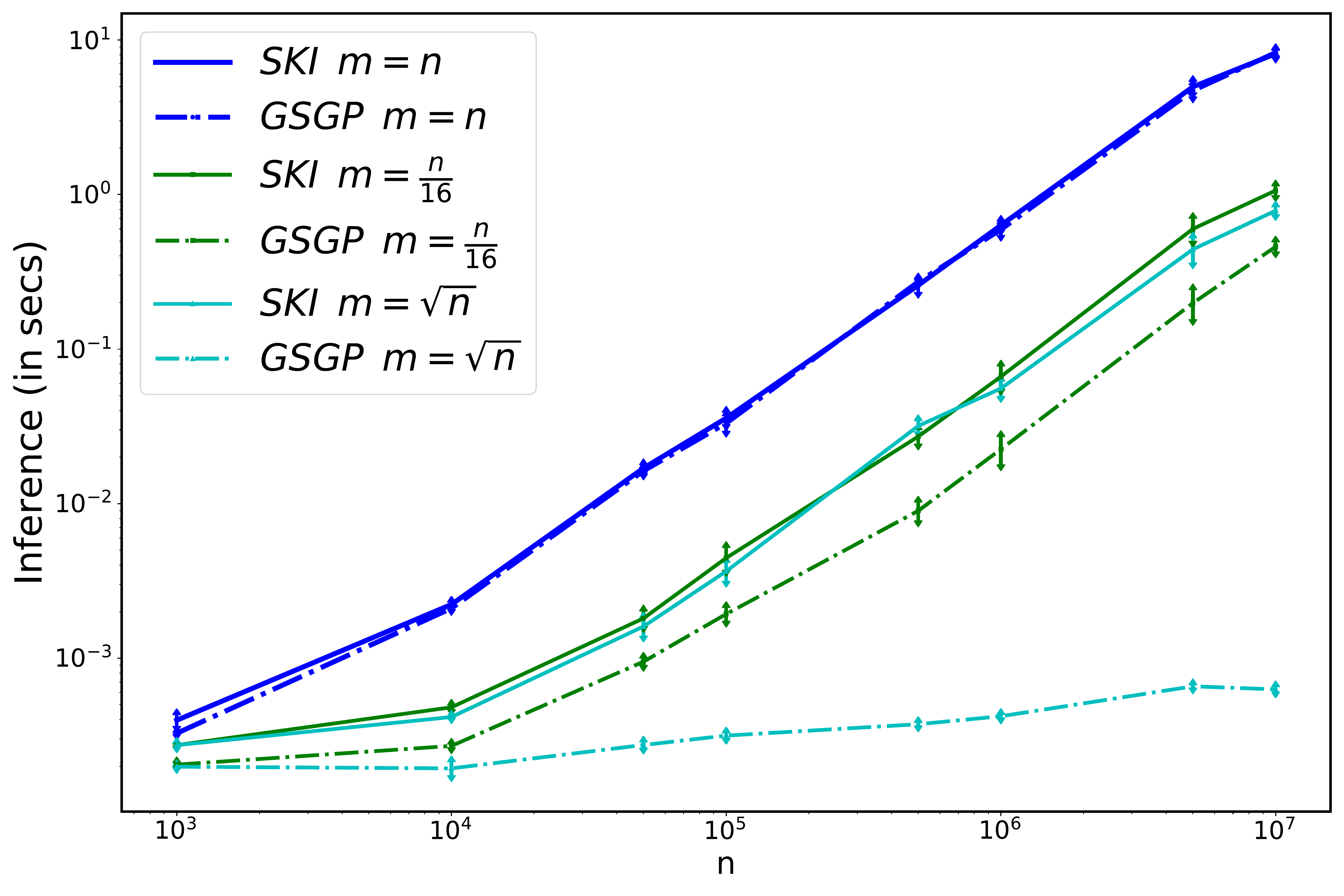}
\caption{\change{Per-iteration time taken by the SKI and GSGP  methods on a synthetic dataset for posterior mean approximation. We see significant speedups when the grid size $m$ is relatively small compared to $n$. Even when $m$ is larger, e.g., $m = \Theta(n)$, GSGP performs no worse than and sometimes improves upon SKI's runtime, e.g., when $m = n/16$.}} 
\label{fig:sine_per_iteration_time_vs_n}
\end{figure}

All linear solves use a tolerance of 0.01, and all kernels are squared exponential. We utilize cubic interpolation for all experiments and provide details on hardware and hyperparameters in Appendix~\ref{app:hyper-parameters and dataset details}. 
In all cases, error bars show 95\% confidence intervals of mean running time over independent trials.

\begin{figure}[h]
\begin{center}
  \setlength\tabcolsep{2pt}
\begin{tabular}[t]{lccccc}
\toprule
Dataset & $n$ & $d$ & $m$ & Time & Memory \\ \hline
\multirow{ 2}{*}{Sound} & 59.3K &  1 & 8K & 0.433 & 0.247    \\  
& 59.3K & 1 &  60K & 0.941 & 0.505  \\ \hline 
\multirow{ 2}{*}{Radar} & 10.5M & 3 & 51.2K & 0.014 & 0.007  \\ 
& 10.5M & 3 & 6.4M & 0.584 & 0.425  \\ \hline 
\multirow{ 4}{*}{Precipitation} & 528K & 3 & 128K & 0.326 & 0.366  \\ 
& 528K & 3 & 528K & 0.491 & 0.628   \\
& 528K & 3 & 1.2M & 0.806 & 2.941  \\ \hline
\end{tabular}

\end{center}
\captionof{table}{\change{Ratios of GSGP to SKI per-iteration time and memory usage for posterior mean approximation for different values of $m$ and $n$. GSGP shows large improvements in a range of settings, even with very large grid size $m$.}}
\label{tab:per_iteration_time_memory_requirements}
\end{figure}

\textbf{Per-iteration resource usage.}
We first compare the per-iteration runtime for posterior mean calculation using CG and Factorized CG on synthetic data of varying sizes. 
The function $f(x)$ is a sine wave with two periods in the interval $[0, 1]$. Random $x$ locations are sampled in the interval and $y = f(x) + \epsilon$ with $\epsilon \sim \mathcal{N}(0, 0.25)$; grid points are equally spaced on $[0, 1]$.
Figure~\ref{fig:sine_per_iteration_time_vs_n} shows the average per-iteration inference time over all iterations of 8 independent trials for increasing $n$ and three different settings of grid size: $m = n$, $m = n/16$ and $m=\sqrt{n}$. 
GSGP is substantially faster when $m < n$ (note log scale) and no slower when $m=n$. 

Memory usage is another important consideration. Table \ref{tab:per_iteration_time_memory_requirements} compares both per-iteration running time and memory usage for posterior mean inference on our real data sets.

Per-iteration time is averaged over one run of CG and FCG for each setting; memory usage is calculated as $\nnz(W) + m + n$ for SKI and $\nnz(W^T W) + 2m$ for GSGP,
where $\nnz(A)$ is the number of nonzero entries of $A$. The gains are significant, especially when $m \ll n$, and gains in time and/or memory are possible even when $m$ equals or exceeds $n$ (e.g., precipitation, $m \in \{528\K, 1.2\M \}$); for very large $m$ it is more resource efficient to run the original algorithm (or use FCG without precomputing $W^T W$).

\begin{figure}[h]
    \begin{center}
    \begin{subfigure}{0.38\textwidth}
    \includegraphics[width=\textwidth]{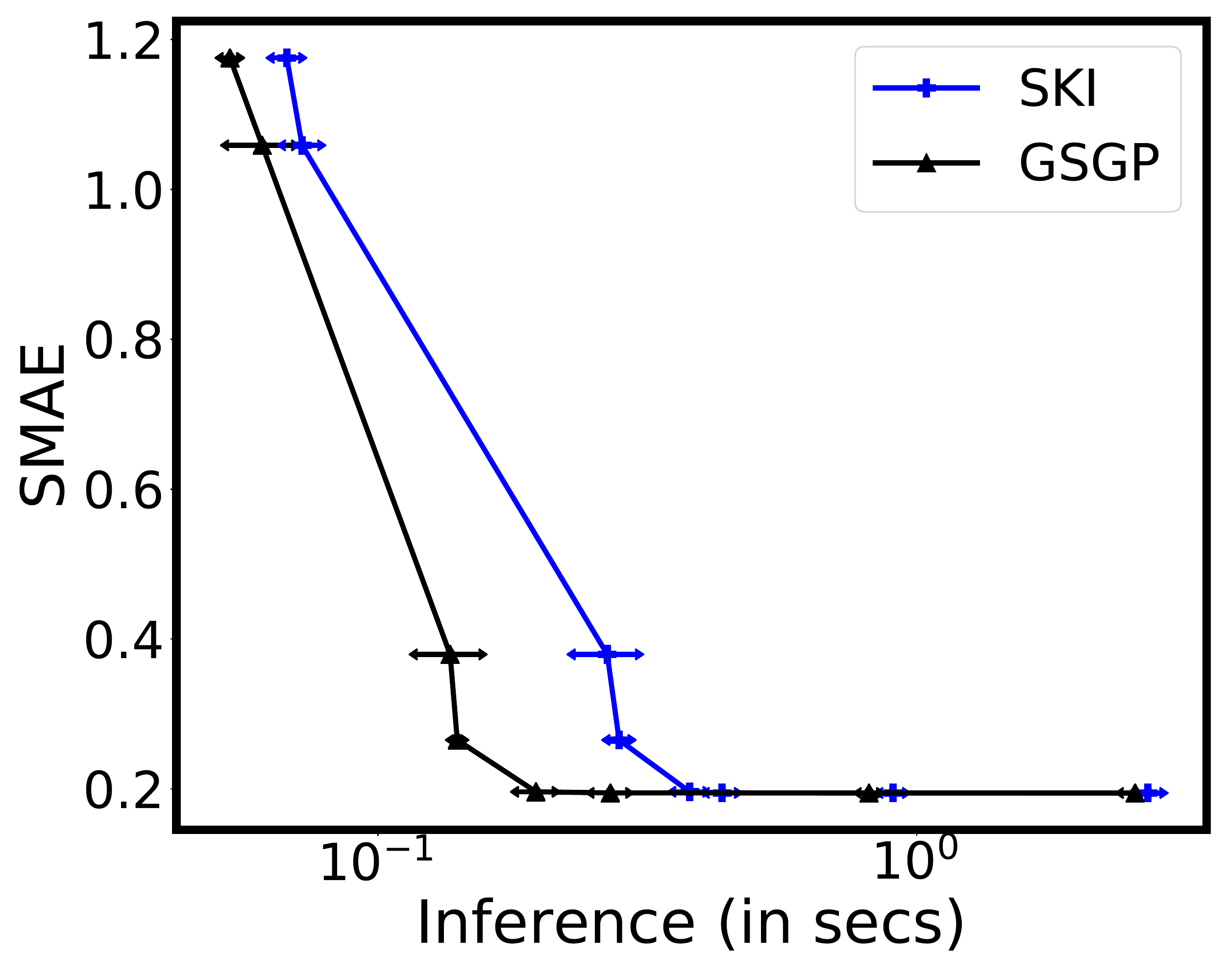}
    \end{subfigure}
    \begin{subfigure}{0.38\textwidth}
    \includegraphics[width=\textwidth]{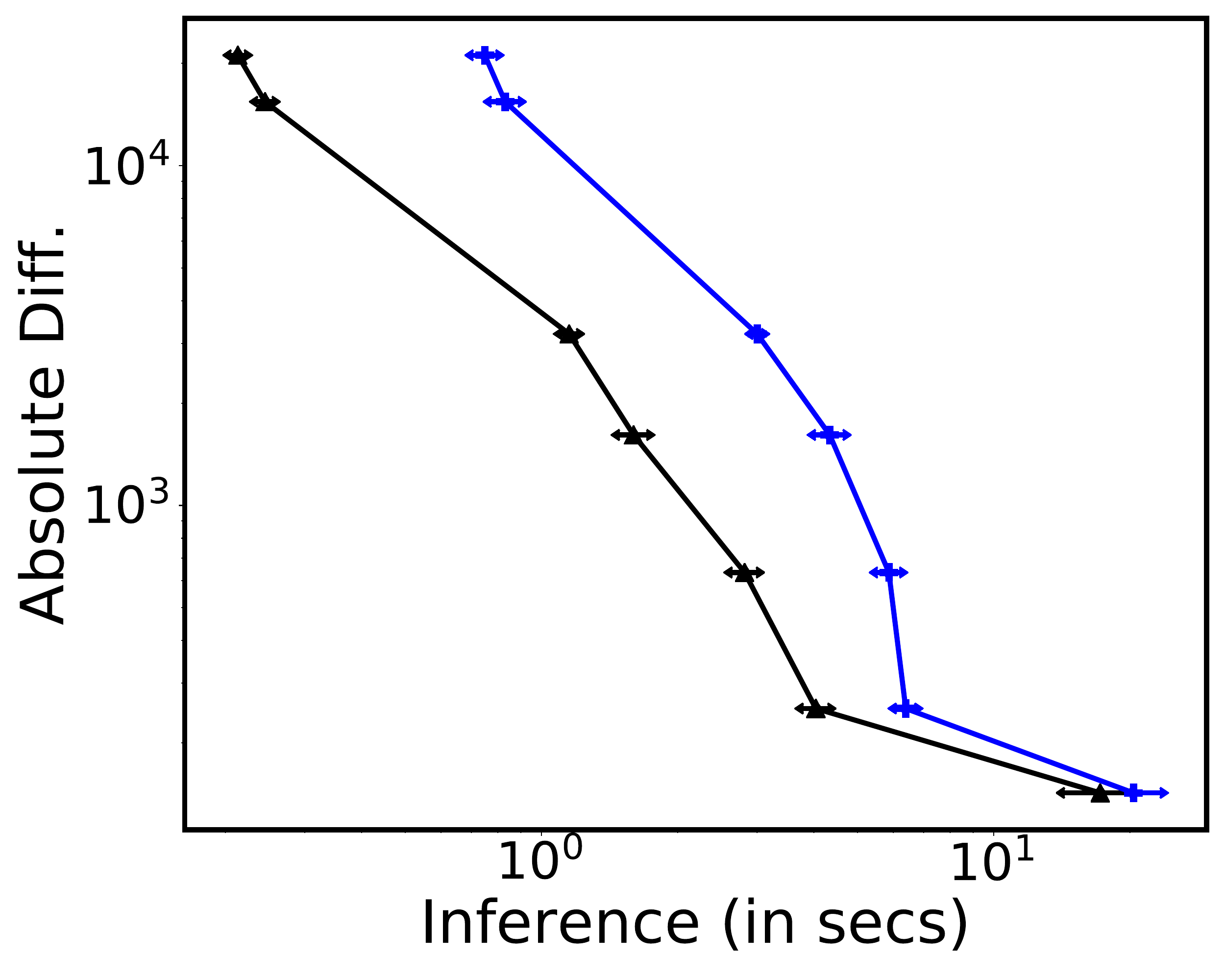}
    \end{subfigure}
    \begin{subfigure}{0.38\textwidth}
    \includegraphics[width=\textwidth]{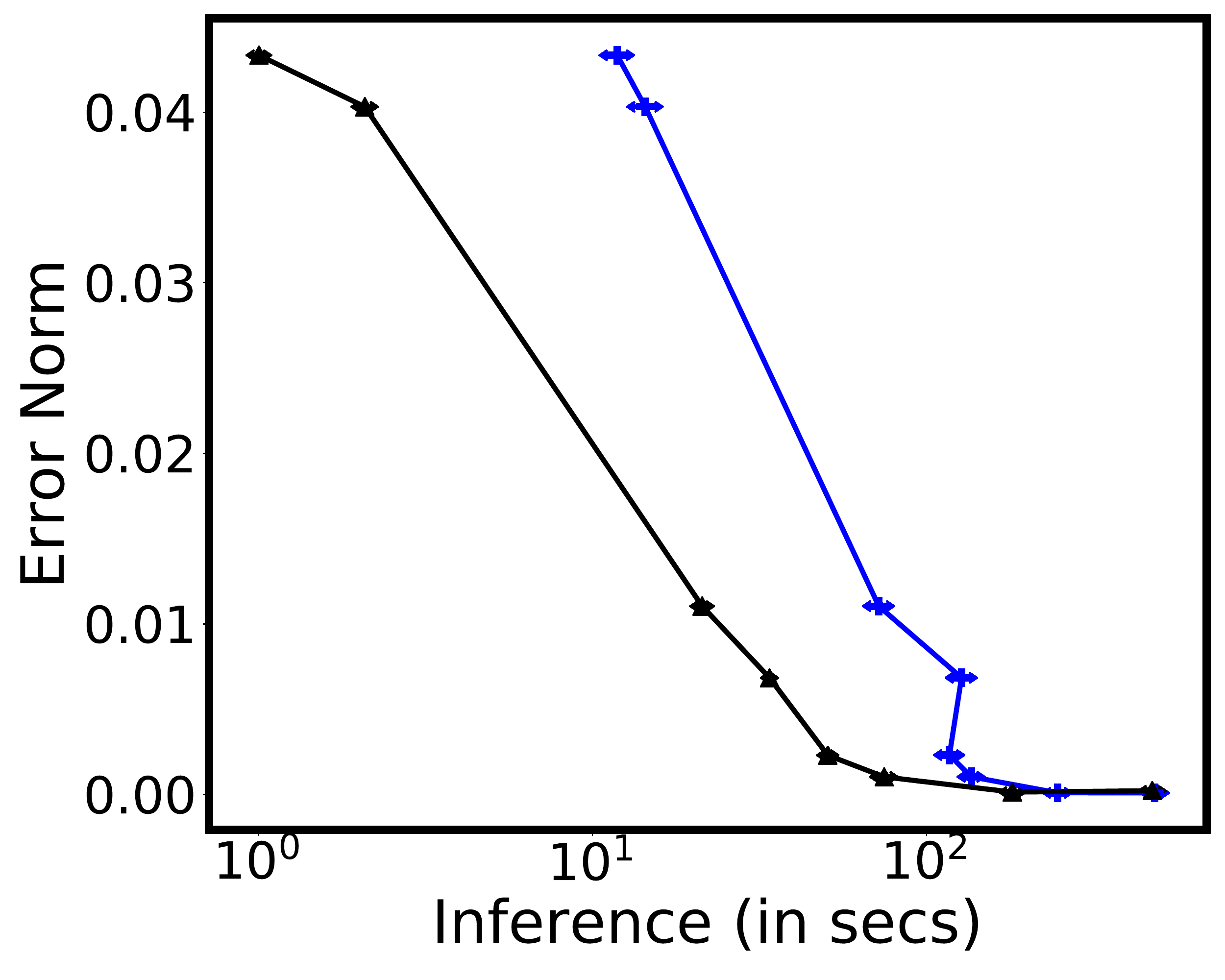}
    \end{subfigure}
    \begin{subfigure}{0.38\textwidth}
    \includegraphics[width=\textwidth]{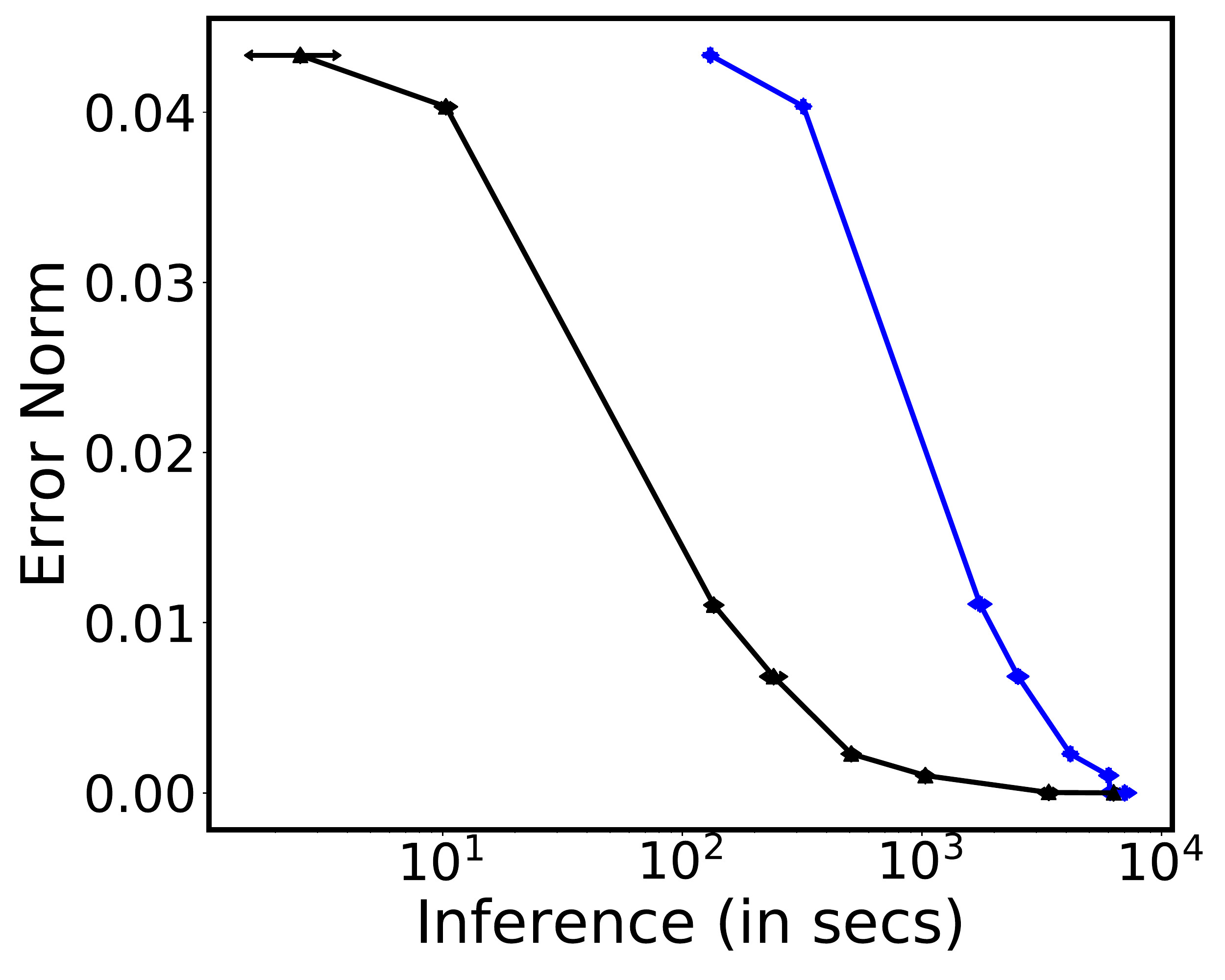}
    \end{subfigure}
    \end{center}
    \caption{\change{Error vs. runtime for approximate inference tasks on the sound dataset with varying grid size $m$. GSGP gives much faster runtimes for fixed $m$, allowing one to use a larger grid and achieve better runtime-accuracy tradeoffs than SKI.
Times are averaged over 20 trials and include pre-processing. See text for description of error metrics.}}

    \label{fig:sound_data_results}
\end{figure}

\textbf{Inference accuracy vs. time}.
A significant advantage of GSGP is the ability to realize speed-accuracy tradeoffs that were not previously possible. Figure~\ref{fig:sound_data_results} illustrates this for the sound data set ($n=59.3\K$) by comparing error vs. running time for four different GP inference tasks for grid sizes  $m \in \{1\K, 2\K, 5\K, 6\K, 8\K, 10\K, 30\K, 60\K \}$. For mean estimation we compute SMAE (mean absolute error normalized by the mean of observations) on a held-out test set of 691 points. For other tasks (log-likelihood, covariance) we compute error relative to a reference value computed with SKI for the highest $m$ using absolute difference for log-likelihood and Frobenius norm from the reference value for covariance matrices.
For log-likelihood, we use 30 samples and $tol=0.01$ for stochastic logdet approximation~\cite{ScalableLD}. For covariance, we compute the $691 \times 691$ posterior covariance matrix for test points, first using the exact SKI expressions (which requires 691 linear solves) and then using a rank-$k$ approximation~\cite{ConstantTimePD}, that, once computed, yields $\O(k)$ time approximations of posterior covariances, for  $k=\min\{m, 10000\}$. 
For each task, GSGP is faster when $m < n$, sometimes substantially so, and achieves the same accuracy, leading to strictly better time-accuracy tradeoffs.

\begin{figure}[h]
\begin{center}
    \begin{subfigure}{0.38 \textwidth}
    \includegraphics[width=\textwidth]{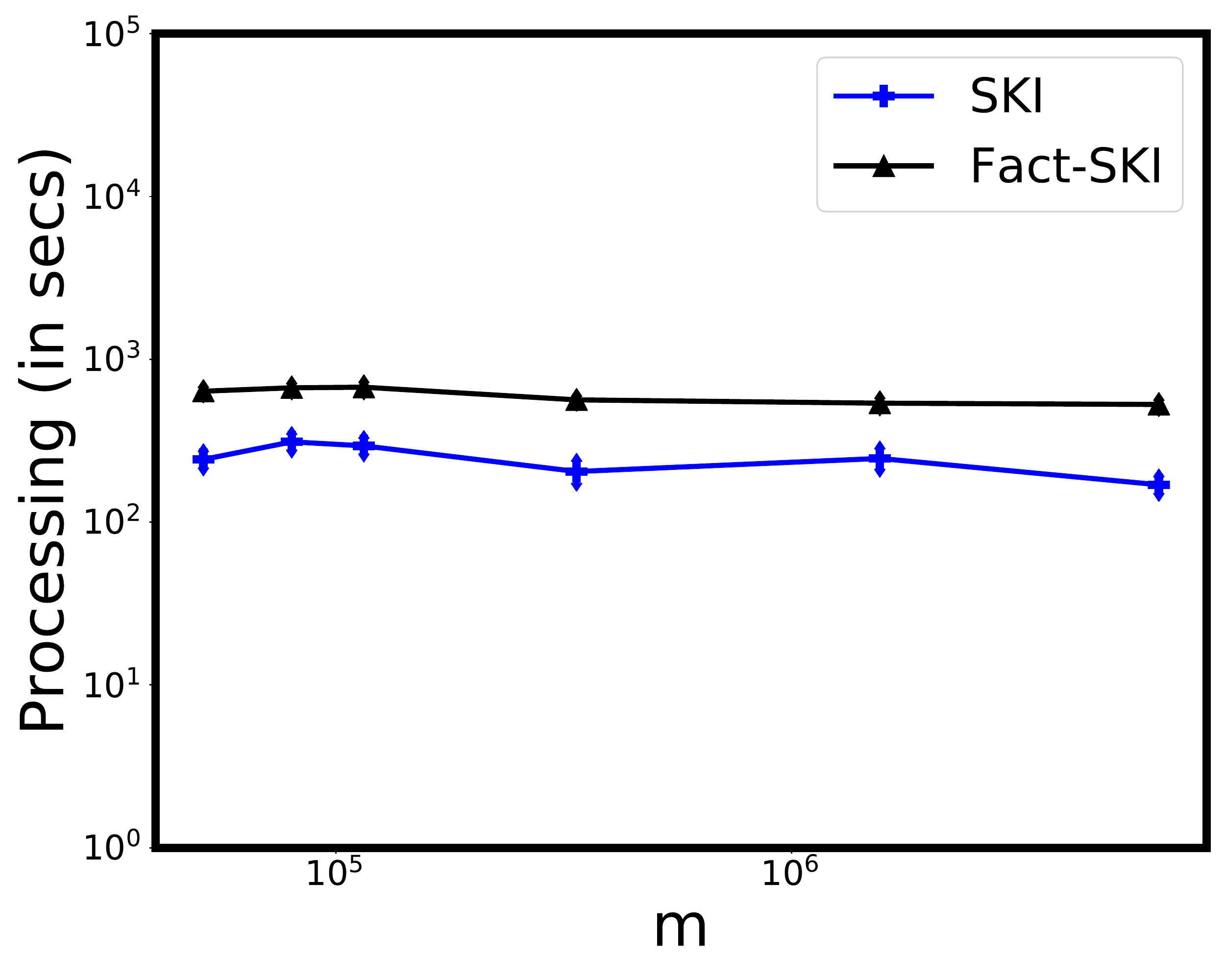}
    \end{subfigure}
    \begin{subfigure}{0.38 \textwidth}
    \includegraphics[width=\textwidth]{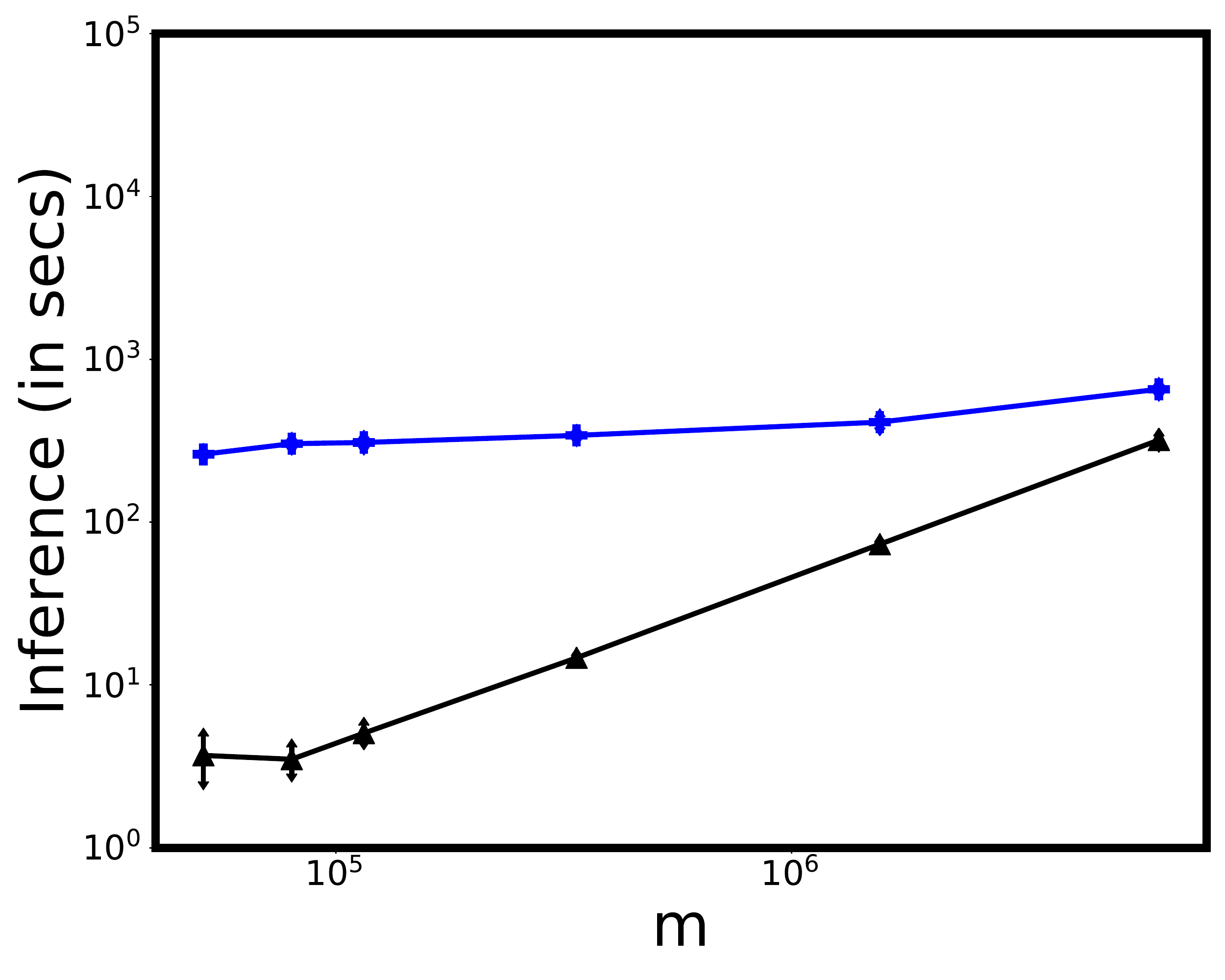}
    \end{subfigure}
    \begin{subfigure}{0.38 \textwidth}
    \includegraphics[width=\textwidth]{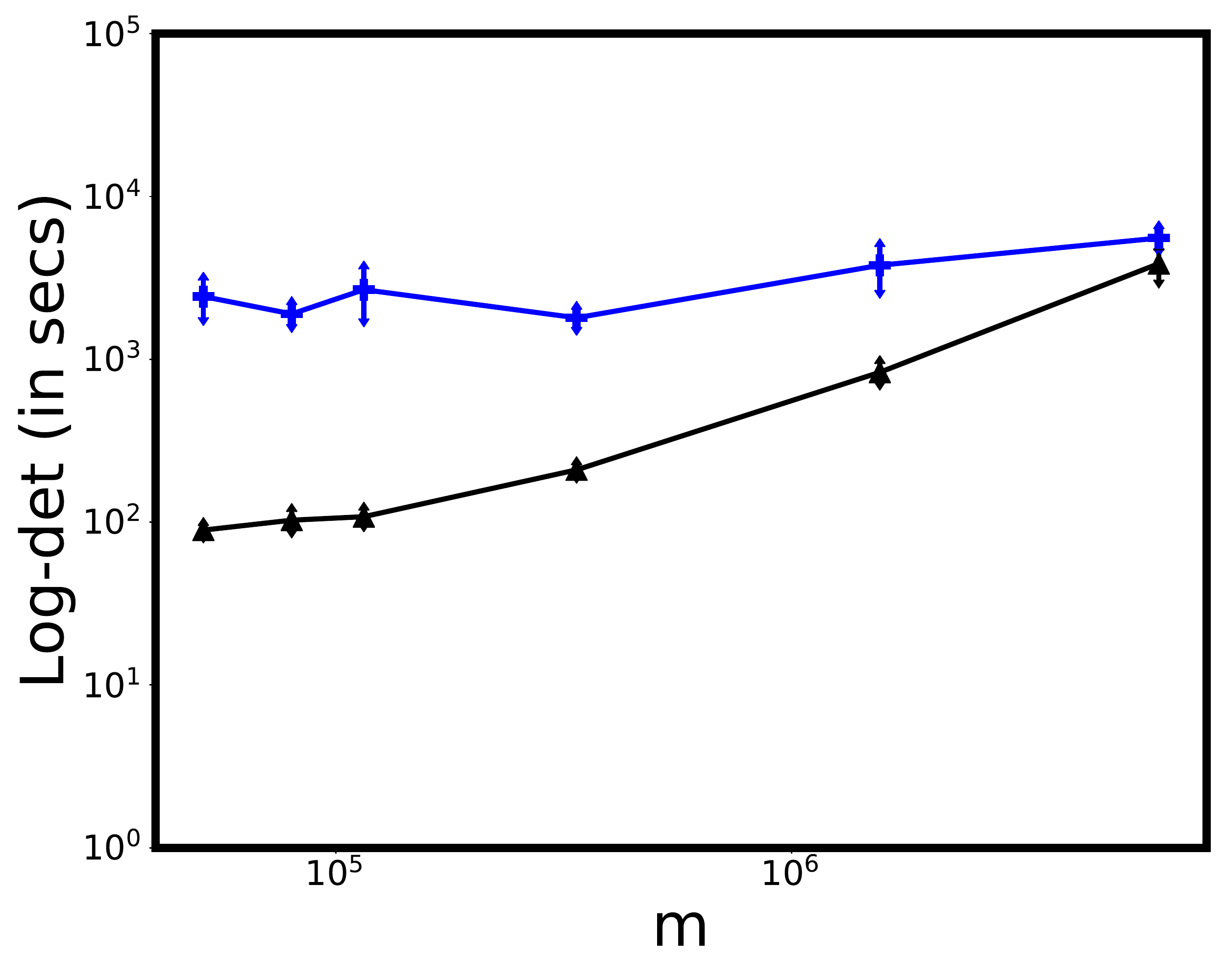}
    \end{subfigure}
    \begin{subfigure}{0.38 \textwidth}
    \includegraphics[width=\textwidth]{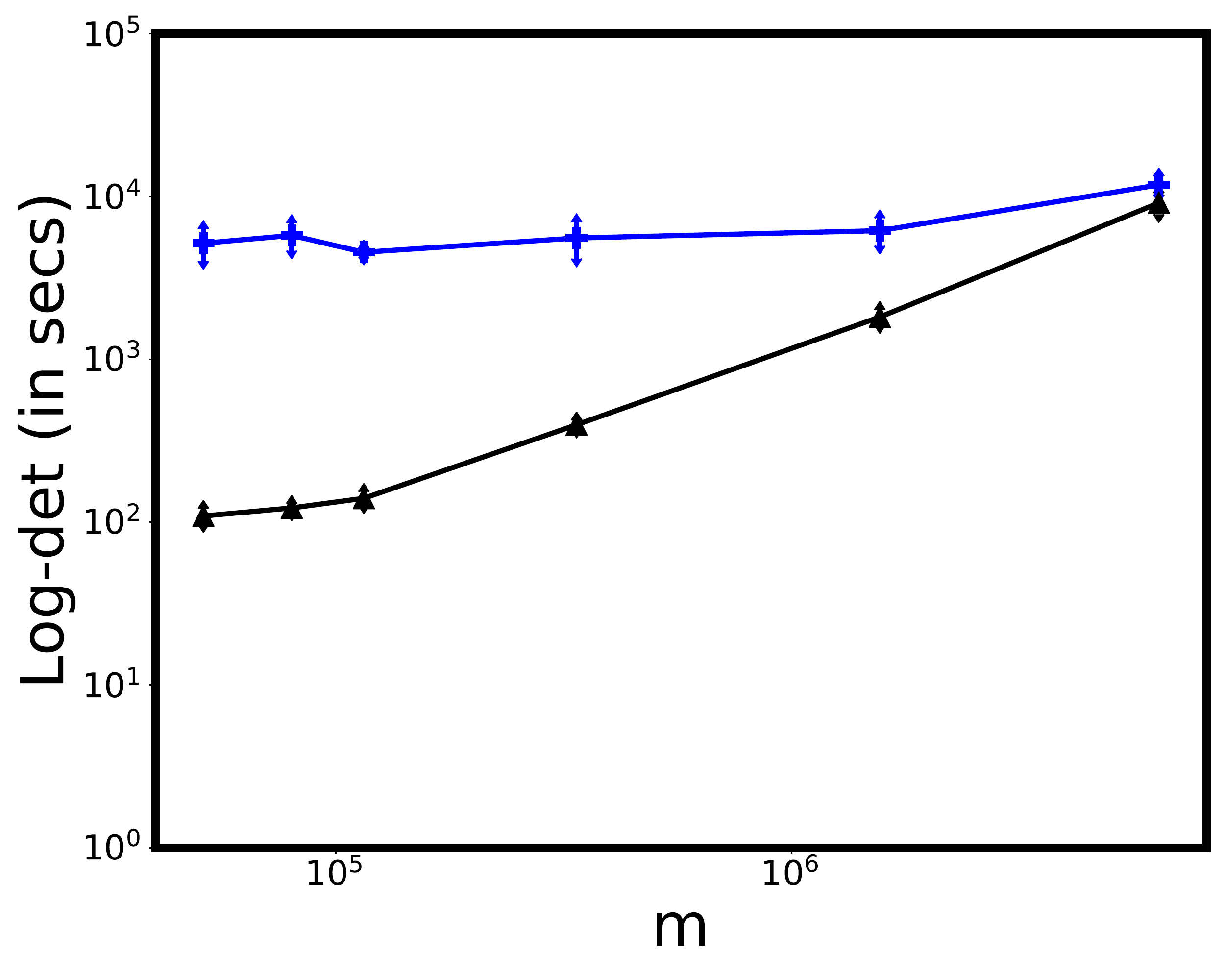}
    \end{subfigure}
    \caption{\change{Runtime vs. grid size $m$ for GP inference tasks on the radar data set with $10.5$M data points. We observe significantly faster runtimes for GSGP across all tasks and a wide range of grid sizes. From top to bottom: pre-processing time, mean inference runtime, and log-determinant runtime for $tol=0.1$ and $tol=0.01$ respectively. 30 random vectors are used in both log-determinant computations.}}
    \label{fig:results_on_radar_dataset}
    \end{center}
\end{figure}

\textbf{Very large $n$}. Figure~\ref{fig:results_on_radar_dataset} shows running time vs. $m$ for GP inference tasks on a data set of $n=10.5\M$ radar reflectivity measurements in three dimensions from 13 radar stations in the northeast US~\cite{angell2018inferring}.
This is a situation where $m \ll n$ is highly relevant: even the \emph{smallest} grid size of $51.2\K = 80 \times 80 \times 8$ is of scientific value for summarizing broad-scale weather and biological activity.
GSGP is much faster, e.g., roughly 150x and 15x faster for $m=51.2K$  on mean inference and log-likelihood estimation respectively after one-time pre-processing (first panel). Pre-processing is up to 3x slower for GSGP due to the need to compute $W^TW$.
To perform only one mean inference, the overall time of GSGP and SKI \emph{including} pre-processing is similar, which is consistent with the observation that typical solves use only tens of iterations, and some of the per-iteration gain is offset by pre-processing. 

However, scientific modeling is highly iterative, and tasks other than mean inference perform \emph{many} more iterations of linear solvers; the total time for GSGP in these cases is much smaller than SKI. In realistic applications with massive data sets, we expect $W^T W$ to be computed once and saved to disk. GSGP also has the significant advantage that its memory footprint is $\O(m)$, while SKI is $\O(m+n)$.
The data above is a \emph{subset} from a national radar network, which was the limit on which SKI could run without exceeding 10GB of memory. To demonstrate scalability of GSGP, we ran on data from the \emph{entire} national radar network with $n=120M$ for $m=128K$, on which SKI far exceeds this memory limit. On this problem, GSGP takes $4861.60$ +/- $233.42$ seconds for pre-processing, and then $9.33$ +/- $0.31$ seconds for mean inference (averaged over 10 trials).

\section{Conclusions and Future Work}
\label{sec:conclusion}

Our work shows that the SKI method for approximate Gaussian process regression in fact performs exact inference for a natural grid-structured Gaussian process. We leverage this observation to give an implementation for  the method with per-iteration complexity \emph{independent of the dataset size $n$}. \change{This leads to significantly improved performance on a range of problems, including the ability to scale GP inference to radar datasets with over $100$ million data points -- a regime far beyond what can typically be handled by SKI or other approximation methods, such as  inducing point and random feature approaches.} 

Our work leaves open a number of questions. Algorithmically, it would be interesting to explore if SKI can be efficiently implemented using direct, rather than iterative, solvers that take advantage of the Toeplitz and band-like structures of $K_G$ and $W^T W$ respectively. 
%
%
Theoretically, it would be interesting to further explore the grid-structured Gaussian process for which SKI performs exact inference. Intuitively, by interpolating data points to a grid, this method seems to suppress `high-frequency' components of the kernel covariance function. Can this be analyzed to lead to formal approximation guarantees or practical guidance in how to choose the grid size?

\bibliographystyle{abbrv}
\bibliography{main} 

\begin{thebibliography}{10}

\bibitem{angell2018inferring}
R.~Angell and D.~R. Sheldon.
\newblock Inferring latent velocities from weather radar data using {G}aussian
  processes.
\newblock In {\em \NIPS{2018}}, pages 8984--8993, 2018.

\bibitem{bishop2006pattern}
C.~M. Bishop.
\newblock {\em Pattern Recognition and Machine Learning}.
\newblock Springer, 2006.

\bibitem{cheng2019bayesian}
Z.~Cheng, M.~Gadelha, S.~Maji, and D.~Sheldon.
\newblock A bayesian perspective on the deep image prior.
\newblock In {\em Proceedings of the IEEE Conference on Computer Vision and
  Pattern Recognition (CVPR)}, pages 5443--5451, 2019.

\bibitem{cressie2008fixed}
N.~Cressie and G.~Johannesson.
\newblock Fixed rank kriging for very large spatial data sets.
\newblock {\em Journal of the Royal Statistical Society: Series B (Statistical
  Methodology)}, 70(1):209--226, 2008.

\bibitem{ScalableLD}
K.~Dong, D.~Eriksson, H.~Nickisch, D.~Bindel, and A.~G. Wilson.
\newblock Scalable log determinants for {G}aussian process kernel learning.
\newblock In {\em \NIPS{2017}}, pages 6327--6337, 2017.

\bibitem{GPyTorchBM}
J.~R. Gardner, G.~Pleiss, D.~Bindel, K.~Q. Weinberger, and A.~G. Wilson.
\newblock {GPyTorch}: Blackbox matrix-matrix {G}aussian process inference with
  {GPU} acceleration.
\newblock In {\em \NIPS{2018}}, pages 7576--7586, 2018.

\bibitem{gardner2018product}
J.~R. Gardner, G.~Pleiss, R.~Wu, K.~Q. Weinberger, and A.~G. Wilson.
\newblock Product kernel interpolation for scalable {G}aussian processes.
\newblock In {\em \AISTATS{2018}}, 2018.

\bibitem{garriga2018deep}
A.~Garriga-Alonso, L.~Aitchison, and C.~E. Rasmussen.
\newblock {Deep Convolutional Networks as Shallow {G}aussian Processes}.
\newblock {\em International Conference on Learning Representations (ICLR)},
  2019.

\bibitem{blockFormulas}
J.~A. Gubner.
\newblock Block matrix formulas, 2015.
\newblock Accessed at:
  \url{https://gubner.ece.wisc.edu/notes/BlockMatrixFormulas.pdf}.

\bibitem{guinness2017circulant}
J.~Guinness and M.~Fuentes.
\newblock Circulant embedding of approximate covariances for inference from
  {G}aussian data on large lattices.
\newblock {\em Journal of Computational and Graphical Statistics},
  26(1):88--97, 2017.

\bibitem{heaton2019case}
M.~J. Heaton, A.~Datta, A.~O. Finley, R.~Furrer, J.~Guinness, R.~Guhaniyogi,
  F.~Gerber, R.~B. Gramacy, D.~Hammerling, M.~Katzfuss, et~al.
\newblock A case study competition among methods for analyzing large spatial
  data.
\newblock {\em Journal of Agricultural, Biological and Environmental
  Statistics}, 24(3):398--425, 2019.

\bibitem{keys1981cubic}
R.~G. Keys.
\newblock Cubic convolution interpolation for digital image processing.
\newblock {\em IEEE Transactions on Acoustics, Speech, and Signal Processing},
  29(6):1153--1160, 1981.

\bibitem{le2013fastfood}
Q.~Le, T.~Sarl{\'o}s, and A.~Smola.
\newblock Fastfood-approximating kernel expansions in loglinear time.
\newblock In {\em \ICML{2013}}, volume~85, 2013.

\bibitem{lee1986fast}
D.~Lee.
\newblock Fast multiplication of a recursive block {T}oeplitz matrix by a
  vector and its application.
\newblock {\em Journal of Complexity}, 2(4):295--305, 1986.

\bibitem{lee2018deep}
J.~Lee, Y.~Bahri, R.~Novak, S.~Schoenholz, J.~Pennington, and
  J.~Sohl-Dickstein.
\newblock {Deep Neural Networks as {G}aussian Processes}.
\newblock {\em International Conference on Learning Representations (ICLR)},
  2018.

\bibitem{matheron1973intrinsic}
G.~Matheron.
\newblock {The Intrinsic Random Functions and Their Applications}.
\newblock {\em Advances in Applied Probability}, 5(3):439--468, 1973.

\bibitem{matthews2018gaussian}
A.~G. d.~G. Matthews, M.~Rowland, J.~Hron, R.~E. Turner, and Z.~Ghahramani.
\newblock Gaussian process behaviour in wide deep neural networks.
\newblock {\em International Conference on Learning Representations (ICLR)},
  2018.

\bibitem{meanti2020kernel}
G.~Meanti, L.~Carratino, L.~Rosasco, and A.~Rudi.
\newblock Kernel methods through the roof: handling billions of points
  efficiently.
\newblock {\em \arXiv{2006.10350}}, 2020.

\bibitem{bayesianRegression}
R.~Neal.
\newblock Lecture notes for sta 414: Statistical methods for machine learning
  and data mining, 2011.
\newblock Accessed at:
  \url{http://www.utstat.utoronto.ca/~radford/sta414.S11/week4a.pdf}.

\bibitem{neal1995bayesian}
R.~M. Neal.
\newblock {\em {Bayesian Learning for Neural Networks}}.
\newblock PhD thesis, University of Toronto, 1995.

\bibitem{novak2018bayesian}
R.~Novak, L.~Xiao, Y.~Bahri, J.~Lee, G.~Yang, J.~Hron, D.~A. Abolafia,
  J.~Pennington, and J.~Sohl-Dickstein.
\newblock {Bayesian Deep Convolutional Networks with Many Channels are
  {G}aussian Processes}.
\newblock In {\em \ICLR{2019}}, 2019.

\bibitem{ConstantTimePD}
G.~Pleiss, J.~R. Gardner, K.~Q. Weinberger, and A.~G. Wilson.
\newblock Constant-time predictive distributions for {G}aussian processes.
\newblock In {\em \ICML{2018}}, pages 4114--4123, 2018.

\bibitem{quinonero2005unifying}
J.~Qui{\~n}onero-Candela and C.~E. Rasmussen.
\newblock A unifying view of sparse approximate {G}aussian process regression.
\newblock {\em Journal of Machine Learning Research}, 6(Dec):1939--1959, 2005.

\bibitem{rahimi2008random}
A.~Rahimi and B.~Recht.
\newblock Random features for large-scale kernel machines.
\newblock In {\em \NIPS{2007}}, pages 1177--1184, 2007.

\bibitem{rasmussen1999evaluation}
C.~E. Rasmussen.
\newblock {\em {Evaluation of {G}aussian Processes and Other Methods for
  Non-linear Regression.}}
\newblock University of Toronto, 1999.

\bibitem{rasmussen2004gaussian}
C.~E. Rasmussen.
\newblock {Gaussian Processes in Machine Learning}.
\newblock In {\em Advanced Lectures on Machine Learning}, pages 63--71.
  Springer, 2004.

\bibitem{saad1986gmres}
Y.~Saad and M.~H. Schultz.
\newblock {GMRES}: A generalized minimal residual algorithm for solving
  nonsymmetric linear systems.
\newblock {\em SIAM Journal on Scientific and Statistical Computing},
  7(3):856--869, 1986.

\bibitem{Snelson2005SparseGP}
E.~Snelson and Z.~Ghahramani.
\newblock Sparse {G}aussian processes using pseudo-inputs.
\newblock In {\em \NIPS{2005}}, pages 1257--1264, 2005.

\bibitem{stroud2017bayesian}
J.~R. Stroud, M.~L. Stein, and S.~Lysen.
\newblock Bayesian and maximum likelihood estimation for {G}aussian processes
  on an incomplete lattice.
\newblock {\em Journal of computational and Graphical Statistics},
  26(1):108--120, 2017.

\bibitem{ubaru2017fast}
S.~Ubaru, J.~Chen, and Y.~Saad.
\newblock Fast estimation of $tr(f(a))$ via stochastic {L}anczos quadrature.
\newblock {\em Journal on Matrix Analysis and Applications}, 38(4):1075--1099,
  2017.

\bibitem{walder2008sparse}
C.~Walder, K.~I. Kim, and B.~Sch{\"o}lkopf.
\newblock Sparse multiscale {G}aussian process regression.
\newblock In {\em \ICML{2008}}, pages 1112--1119, 2008.

\bibitem{williams1997computing}
C.~K. Williams.
\newblock Computing with infinite networks.
\newblock In {\em \NIPS{1997}}, pages 295--301, 1997.

\bibitem{williams1998bayesian}
C.~K. Williams and D.~Barber.
\newblock {Bayesian Classification with {G}aussian Processes}.
\newblock {\em IEEE Transactions on Pattern Analysis and Machine Intelligence
  (PAMI)}, 20(12):1342--1351, 1998.

\bibitem{williams2001using}
C.~K. Williams and M.~Seeger.
\newblock Using the {N}ystr{\"o}m method to speed up kernel machines.
\newblock In {\em \NIPS{2001}}, pages 682--688, 2001.

\bibitem{williams1996gaussian}
C.~K.~I. Williams and C.~E. Rasmussen.
\newblock {Gaussian Processes for Regression}.
\newblock In {\em \NIPS{1996}}, pages 514--520, 1996.

\bibitem{wilson2013gaussian}
A.~Wilson and R.~Adams.
\newblock Gaussian process kernels for pattern discovery and extrapolation.
\newblock In {\em \ICML{2013}}, pages 1067--1075, 2013.

\bibitem{kissgp}
A.~Wilson and H.~Nickisch.
\newblock Kernel interpolation for scalable structured {G}aussian processes
  {(KISS-GP)}.
\newblock In {\em \ICML{2015}}, pages 1775--1784, 2015.

\bibitem{wilson2015thoughts}
A.~G. Wilson, C.~Dann, and H.~Nickisch.
\newblock Thoughts on massively scalable gaussian processes.
\newblock {\em \arXiv{1511.01870}}, 2015.

\end{thebibliography}
\clearpage

\appendix
\begin{appendices}
\begin{center} \Large \textbf{Supplementary Appendices}
\end{center}

\section{Reformulation of SKI as Bayesian Linear Regression -- Omitted Details}\label{app:gsgp}

\subsection{Equivalence between GSGP and SKI Approximation}

\begin{repclaim}{claim:gsgpski}
  A GSGP with grid weights $\boldsymbol{\theta}$ drawn from a GP with covariance function $k(\bv{x}, \bv{x}')$ is itself a GP with covariance function $\tilde{k}(\bv x, \bv x') = \bv{w}_{\bv{x}}^T K_G \bv{w}_{\bv{x}'}$.
\end{repclaim}

\begin{proof}
  For any finite number of input points $\x_1, \ldots, \x_n$, the joint distribution of $(f(\x_1), \ldots, f(\x_n))$ is zero-mean Gaussian, since each $f(\x_i) = \w_{\x_i}^T \boldsymbol{\theta}$ is a linear transformation of the same zero-mean Gaussian random variable $\boldsymbol{\theta}$.
Therefore, $f(\x)$ is a zero-mean GP.
For a particular pair of function values $f(\x)$ and $f(\x')$, the covariance is given by
\begin{equation}
  \cov(f(\x), f(\x')) = \cov(\w_\x^T \boldsymbol{\theta}, \w_{\x'}^T \boldsymbol{\theta}) = \w_\x^T \cov(\boldsymbol{\theta}, \boldsymbol{\theta}) \w_{\x'} = \w_\x^T K_G \w_{\x'} = \tilde{k}(\x, \x').
  \label{eq:gsgpcov}
\end{equation}
Therefore $f$ has the claimed covariance function.
\end{proof}

\subsection{Exact Inference for GSGP}

For completeness we derive the exact GSGP inference equations for Fact \ref{fact:gsgp}.

\begin{proof}[Proof of Fact \ref{fact:gsgp}]
GSGP is a standard linear basis function model and it is well known (see e.g., \cite{bayesianRegression}) that the posterior mean and covariance of $\bs{\theta}$ given $\bv{y}$ are given by
\begin{align}\label{eq:thetaMean}
\E[\boldsymbol{\theta} | \bv y]  = \frac{1}{\sigma^2} \cdot \left [ K_G^{-1} + \frac{1}{\sigma^2} W^T W \right ]^{-1} \cdot W^T \bv{y}.
\end{align}
and 
\begin{align}\label{eq:thetaVar}
\var(\boldsymbol{\theta} | \bv y) =  \left [ K_G^{-1} + \frac{1}{\sigma^2} W^T W \right]^{-1}.
\end{align}
Using that for invertible $A,B$, $(AB)^{-1} = B^{-1} A^{-1}$ we can write: 
\begin{align*}
\frac{1}{\sigma^2} \cdot \left [ K_G^{-1} + \frac{1}{\sigma^2} W^T W \right ]^{-1} = \frac{1}{\sigma^2} \cdot \left [ I + \frac{1}{\sigma^2} K_G W^T W \right ]^{-1} K_G = (K_G W^T W + \sigma^2 I)^{-1} K_G.
\end{align*} 
Plugging back into the posterior mean equation \eqref{eq:thetaMean} gives
\begin{align*}
\E[\boldsymbol{\theta} | \bv y] = (K_G W^T W + \sigma^2 I)^{-1} K_GW^T \bv{y} = \bv{\bar z}.
\end{align*} 
By linearity since $f(\bv{x}) = \bv w_{\bv{x}}^T \bs{\theta}$, this gives the GSGP posterior mean equation, $\mu_{f|\mathcal{D}}(\mathbf{x}) = \bv w_{\bv{x}}^T  \bv{\bar z}$.
Similarly, plugging back into the posterior variance equation \eqref{eq:thetaVar} gives:
\begin{align*}
\var(\boldsymbol{\theta} | \bv y) = \sigma^2 (K_G W^T W + \sigma^2 I)^{-1}K_G = \bv{\bar C}.
\end{align*}
Again by linearity this gives the GSGP posterior variance equation $ k_{f|\mathcal{D}}(\mathbf{x}, \mathbf{x}^{\prime}) =  \bv{w}_\bv{x}^T \bv{\bar{C}} \bv{w}_\bv{x'}$.

Finally, it is well known \cite{bayesianRegression} that the log likelihood of $\bv{y}$ under the GSGP linear basis function model is given by:
\begin{align*}
&\log \Pr(\bv{y}) = -\frac{1}{2} \left [ n \log (2\pi) + n \log (\sigma^2) + \log \det K_{G} -  \log \det(\bv{\bar C}) + \frac{1}{\sigma^2} \bv{y}^T \bv{y} - \bv{\bar z}^T \bv{\bar C}^{-1} \bv{\bar z}  \right  ]
\end{align*}
We can split $\log \det(\bv{\bar C}) = \log \det(\sigma^2 (K_{G} W^T W + \sigma^2 I)^{-1} K_{G}) = m \log \sigma^2 - \log \det(K_{G} W^T W + \sigma^2 I) + \log\det(K_{G})$, and canceling this gives:
\begin{align}\label{eq:splitGP}
&\log \Pr(\bv{y}) = -\frac{1}{2} \left  [ n \log (2\pi) + (n-m)\log \sigma^2 + \log \det(K_{G} W^T W + \sigma^2 I) + \frac{1}{\sigma^2} \bv{y}^T \bv{y} - \bv{\bar z}^T \bv{\bar C}^{-1} \bv{\bar z}    \right  ]
\end{align}
Additionally we can write:
\begin{align*}
 \bv{\bar z}^T \bv{\bar C}^{-1} \bv{\bar z}   &= \frac{1}{\sigma^2} \bv{\bar z}^T  K_G^{-1} (K_G W^T W + \sigma^2 I)  (K_G W^T W + \sigma^2 I)^{-1}  K_G W^T \bv{y} =\frac{1}{\sigma^2} \bv{\bar z}^T W^T  \bv{y}.
\end{align*}
Thus we have $ \frac{1}{\sigma^2} \bv{y}^T \bv{y} - \bv{\bar z}^T \bv{\bar C}^{-1} \bv{\bar z} = \frac{\bv{y}^T (\bv{y} - W\bv{\bar z})}{\sigma^2} $. Plugging back into \eqref{eq:splitGP} yields the log likelihood formula of Fact \ref{fact:gsgp}, completing the proof.
\end{proof}

\subsection{Equivalence of GSGP Exact Inference and SKI}

Theorem \ref{thm:gsgp}, which states that the GSGP and SKI inference equations are identical, follows directly from Claim \ref{claim:gsgpski}, as GSGP is a Gaussian process whose kernel is exactly the approximate kernel used by SKI. 
For illustrative purposes, we also give a purely linear algebraic proof of Theorem \ref{thm:gsgp} below:
\begin{reptheorem}{thm:gsgp}
The inference expressions of Fact \ref{fact:gsgp} are identical to the SKI approximations of Def. \ref{def:kissgp}.
\end{reptheorem}
\begin{proof}[Linear Algebraic Proof of Theorem \ref{thm:gsgp}]
  The equivalence between the expressions for the posterior mean and posterior variance is a standard manipulation to convert between the ``weight space'' and ``function space'' view of a GP with an explicit feature expansion, e.g., Eqs. (2.11), (2.12) of \cite{rasmussen2004gaussian}. We provide details with our notation for completeness. The correspondence between the two expressions for the log-likelihood is due to the same correspondence between ``weight space'' and ``function space'' views, though we are not aware of a specific reference that provides the formula in Fact~\ref{fact:gsgp}.

First, recall these definitions from Def. \ref{def:kissgp} and Fact \ref{fact:gsgp}:
\begin{align*}
  \tilde{K}_X &= W K_G W^T \\
  \bv{\tilde z} &= \left (WK_GW^T + \sigma^2 I\right )^{-1} \mathbf{y} \\
  \bv{\bar z} &= (K_{G} W^T W + \sigma^2 I)^{-1} K_{G} W^{T} \bv y 
\end{align*}

\textbf{Mean:}
The two mean expressions are $\w_\x K_G W^T \bv{\tilde z}$ (Def. \ref{def:kissgp}) and $\w_\x \bv{\bar z}$ (Fact \ref{fact:gsgp}). Expanding these, it suffices to show that
\begin{align}\label{eq:meanShrink}
K_{G} W^T \left (W K_{G}W ^{T} + \sigma^2 I\right )^{-1} = \left (K_{G}W^{T} W + \sigma^2 I\right )^{-1} K_{G} W^T.
\end{align} This follows from the following claim:
\begin{claim}\label{clm:covEquivilance}
Let $A \in \R^{m \times n}$ and $B \in \R^{n \times m}$. Then
$$
A(BA + \sigma^2 I_n)^{-1} = (AB + \sigma^2 I_m)^{-1}A
$$
as long as $BA + \sigma^2 I_n$ and $AB + \sigma^2 I_m$ are both invertible.
\label{claim:identity}
\end{claim}
\begin{proof}
Observe that
$
(AB + \sigma^2 I_m) A = A(BA + \sigma^2 I_n) = ABA + \sigma^2 A.
$
If the matrices $AB + \sigma^2 I_m$ and $BA + \sigma^2 I_n$ are invertible the result follows by left- and right-multiplying by the inverses.
\end{proof}
We obtain \eqref{eq:meanShrink} from Claim \ref{clm:covEquivilance}, applied with $A = K_{G}W ^{T}$ and $B = W$. As required by the claim, we have that both $K_{G}W ^{T} W + \sigma^2 I$ and $W K_{G}W ^{T} + \sigma^2 I$ are invertible. For $\sigma > 0$, $W K_{G}W ^{T} + \sigma^2 I$ is positive definite, which implies invertibility. $K_{G}W ^{T} W$ is similar to the positive semidefinite matrix $K_G^{1/2} W^T W K_G^{1/2}$, and thus has all non-negative eigenvalues. Thus, for $\sigma > 0$, $K_{G}W ^{T} W + \sigma^2 I$ has all positive eigenvalues (in particular, it has no zero eigenvalue) and so is invertible.

\textbf{Covariance:}
The two expressions for covariance are $\w_x \bv{\tilde C} \w_x$ (Def.~\ref{def:kissgp}) and $\w_\x \bv{\bar C} \w_x$ (Fact.~\ref{fact:gsgp}) with
\begin{align*}
  \bv{\tilde C} &= K_{G} - K_G W^T (W K_G W^T +\sigma^2 I)^{-1} W K_G \\
  \bv{\bar C} &= \sigma^2 \left( K_G W^{T}W   + \sigma^2 I \right)^{-1}  K_{G}
\end{align*}
So it suffices to show that $\bv{\tilde C} = \bv{\bar C}$. Using Eq.~\ref{eq:meanShrink}, we have that
$$
\bv{\tilde C} = K_{G} - (K_{G} W^T W+\sigma^2 I)^{-1} K_{G} W^T W K_{G}.
$$
Now, factor out $K_G$ and simplify to get:
\begin{align*}
\bv{\tilde C} &= [I - (K_{G} W^T W+\sigma^2 I)^{-1} K_{G}W^T W] K_{G}\\
&=  [(K_{G} W^T W+\sigma^2 I)^{-1} \cdot (K_{G} W^T W+\sigma^2 I - K_{G} W^T W)] K_{G}\\
&= \sigma^2 (K_G W^T W + \sigma^2 I)^{-1} K_{G} = \bv{\bar C}
\end{align*}

\textbf{Log likelihood:}
By matching terms in the two expressions and using the definition of $\tilde{K}_X$, it suffices to show both of the following:
\begin{align}
\bv{y}^T \bv{\tilde z} &= \frac{\bv{y}^T (\bv{y} - W \bv{\bar z})}{\sigma^2}, \label{logsuffices0}\\
\log\det(W K_G W^T + \sigma^2 I) &= \log\det(K_GW^T W + \sigma^2 I) + (n-m) \log \sigma^2.  \label{logsuffices1}
\end{align}
For Eq.~\eqref{logsuffices0}, we first observe that by our argument above that $\bv{\bar z} = K_G W^T \bv{\tilde z}$, we have
\begin{align}\label{logsuffices2}
 \frac{\bv{y}^T (\bv{y} - W \bv{\bar z})}{\sigma^2} = \frac{\bv{y}^T (\bv{y} - W K_G W^T \bv{\tilde z})}{\sigma^2} = \frac{\bv{y}^T (I - W K_G W^T (W K_G W^T + \sigma^2  I)^{-1}) \bv{y}}{\sigma^2}.
 %
\end{align}
We have:
\begin{align*}
(I - W K_G W^T (W K_G W^T + \sigma^2  I)^{-1}) &= [W K_G W^T + \sigma^2  I - W K_G W^T] (W K_G W^T + \sigma^2  I)^{-1}\\
& = \sigma^2 (W K_G W^T + \sigma^2  I)^{-1}.
\end{align*}
Thus we can simplify \eqref{logsuffices2} to:
\begin{align}\label{logsuffices3}
 \frac{\bv{y}^T (\bv{y} - W \bv{\bar z})}{\sigma^2} = \bv{y}^T (W K_G W^T + \sigma^2  I)^{-1} \bv{y} = \bv{y}^T \bv{\tilde z}.
\end{align}
It remains to prove Eq.~\ref{logsuffices1}. This follows from the general claim:

\begin{claim}\label{clm:det} Let $A \in \R^{m \times n}$ and $B \in \R^{n \times m}$. Then:
$$
\det(BA + \sigma^2 I_n) = (\sigma^2)^{n-m} \det(AB + \sigma^2 I_m)
$$
\end{claim}
\begin{proof}
The claim generalizes the Weinstein-Aronszajn identity, which states that $\det(BA + I_n) = \det(AB + I_m)$. It can be proven using the block determinant formulas in \cite{blockFormulas}. 
Let $C = \begin{bmatrix} \sigma I_m & -A\\ 
B & \sigma I_n
\end{bmatrix}.$
We have:
\begin{align*}
\det(C) = \det(\sigma I_m) \cdot \det(B (\sigma I_m)^{-1} A + \sigma I_n) &= \det(\sigma I_m) \cdot \det \left (\frac{1}{\sigma} (B A + \sigma^2 I_n) \right )\\
& = \det(\sigma I_m) \cdot \det \left (\frac{1}{\sigma} I_n \right ) \cdot \det(B A + \sigma^2 I_n)
\end{align*}
and similarly
\begin{align*}
\det(C) = \det(\sigma I_n) \cdot \det(A  (\sigma I_n)^{-1}  B + \sigma I_m) &= \det(\sigma I_n) \cdot \det \left (\frac{1}{\sigma} (A  B + \sigma^2 I_m)\right ) \\
& = \det(\sigma I_n) \cdot \det \left (\frac{1}{\sigma} I_m \right ) \cdot \det (A  B + \sigma^2 I_m).
\end{align*}
 Thus, 
 \begin{align*}
\det(\sigma I_m) \cdot \det \left (\frac{1}{\sigma} I_n \right ) \cdot \det(BA + \sigma^2 I_n) &=  \det(\sigma I_n) \cdot \det \left (\frac{1}{\sigma} I_m \right ) \cdot \det(AB + \sigma^2 I_m)\\
\det(B A + \sigma^2 I_n) &= \sigma^{2(n-m)} \cdot \det(AB + \sigma^2 I_m),
 \end{align*}
 giving the claim.
\end{proof}
Applying Claim \ref{clm:det} with $A = K_G W^T$ and $B = W$ gives that $\det(\tilde W K_G W^T + \sigma^2 I) = (\sigma^2)^{n-m} \det(K_GW^T W + \sigma^2 I) $ which in turn gives that $\log \det(W K_G W^T + \sigma^2 I) = \log\det(K_GW^T W + \sigma^2 I) + (n-m)\log \sigma^2$ and so completes Eq.~\eqref{logsuffices1} and the theorem.
\end{proof}

\subsection{Efficiency of Computing and Multiplying $W^TW$}

\begin{repclaim}{claim:WTW}
    Assume that $G = \{\bv{g}_1,\ldots,\bv{g}_m\}$ has spacing $s$, i.e., $\norm{\bv{g}_i-\bv{g}_j}_\infty \ge s$ for any $i,j \in m$. Also assume that $\w^j_{\x}$ is non-zero only if $\|\g_j - \x\|_\infty < r \cdot s$ for some fixed integer $r$. Then $W^T W$ can be computed in $\mathcal{O}(n(2r)^{2d})$ time and has at most $(4r-1)^{d}$ entries per row. Therefore $mv(W^TW) = \mathcal{O}(m (4r-1)^d).$
\end{repclaim}

\begin{proof}
First, write $W^TW$ as a sum over data points:
$$W^T W = \sum_{i=1}^n \w_{\x_i} \w_{\x_i}^T.$$

The vector $\w_{\x_i}$ has nonzeros only for grid points $\g_j$ such that $\|\g_j - \x_i\|_{\infty} \le r \cdot s$. Since the grid points have spacing $s$, there are at most $(2r)^d$ such grid points. Therefore the outer product $\w_{\x_i} \w_{\x_i}^T$ has at most $(2r)^{2d}$ non-zeros, and the sum can be computed in $\mathcal{O}(n(2r)^{2d})$ time. 

The sum is also sparse. The entry $(W^T W)_{jk}$ is non-zero only if there is \emph{some} data point $\x_i$ within distance $r$ from both $\g_j$ and $\g_k$ in each dimension. This is true only if $\|\g_j - \g_k\|_\infty < 2r \cdot s$. For a given grid point $\g_j$, there are at most $(4r-1)^d$ grid points $\g_k$ satisfying $\|\g_j - \g_k\|_\infty < 2r \cdot s$ (e.g., in 1 dimension there are $2r-1$ neighbors to the left, $2r-1$ neighbors to the right, plus the case $k=j$). Therefore the $j$th row of $W^TW$ has as most $(4r-1)^d$ nonzeros, as claimed.
\end{proof}

\subsection{Symmetric Reformulation of GSGP}

Recall that directly replacing the SKI approximate  inference equations of Definition \ref{def:kissgp} with the GSGP inference equations of Fact \ref{fact:gsgp} reduces per-iteration cost from $\O(n +m \log m)$ to $\O(m \log m)$. However, the matrix $K_G W^T W + \sigma^2 I$ is \emph{asymmetric}, which prevents the application of symmetric system solvers like the conjugate gradient method with strong guarantees.

Here we describe a symmetric reformulation of GSGP that doesn't compromise on per-iteration cost.
We write the matrix $(K_G W^T W + \sigma^2 I)^{-1} $, whose application is required in both the posterior mean and covariance computation as:
\begin{align*}
 (K_{G} W^T W + \sigma^2 I)^{-1} &=  (K_{G} W^T W  K_G  K_G^{-1} + \sigma^2 K_G  K_G^{-1}   )^{-1}  \\
&= K_G  (K_{G} W^T W  K_G  + \sigma^2 K_G  )^{-1} .
\end{align*}
Applying the above matrix requires a symmetric solve in $(K_{G} W^T W  K_G  + \sigma^2 K_G  )^{-1} $, along with a single $\O(m \log m)$ time MVM with $K_G$. Our per-iteration complexity thus remains at $\O(m \log m)$ -- the matrix vector multiplication time for $K_{G} W^T W  K_G  + \sigma^2 K_G$. However, this matrix is no longer of the `regularized form' $A + \sigma^2 I$, and may have worse condition number than $W K_G W^T + \sigma^2 I$, possibly leading to slower convergence of iterative solvers like CG as compared to SKI.

We can similarly  symmetrize the logdet computation in the likelihood expression by writing
$$\log\det(K_G W^T W + \sigma^2 I) = \log\det(K_G W^T W K_G + \sigma^2 K_G) - \log\det(K_G).$$
Again however, it is unclear how this might affect the convergence of iterative methods for logdet approximation

\section{Factorized Iterative Methods -- Omitted Details}\label{app:factorized}

\subsection{Proofs for Factorized Update Steps}

We give proofs of our fundamental factorized update steps described in Claims \ref{clm:goldenRule} and \ref{clm:goldenRule2}. 

\begin{repclaim}{clm:goldenRule}
For any $\bv{z}_i \in \R^n$ with $\bv{z}_i = W \bv{\hat z}_i + c_i \bv{z}_0$, 
\begin{align*}
(WK_GW^T + \sigma^2 I) \bv{z}_i = W \bv{\hat z}_{i+1} + c_{i+1} \bv{z}_0,
\end{align*}
where $\bv{\hat z}_{i+1} = (K_G W^T W + \sigma^2 I) \bv{\hat z}_{i} + c_i K_G W^T \bv{z}_0$ and $c_{i+1} = \sigma^2 \cdot c_i$.
Call this operation a \emph{factorized update} and denote it as $(\bv{\hat z}_{i+1}, c_{i+1}) = \mathcal{A}(\bv{\hat z}_{i},c_i)$.
If the vector $K_G W^T \bv{z}_0$ is precomputed in $\mathcal{O}(n + m \log m)$ time, each subsequent factorized update takes $\mathcal O(m \log m)$ time.
\end{repclaim}
 \begin{proof}
  We have:
  \begin{align*}
  (WK_GW^T + \sigma^2 I) \bv{z}_i  &= (WK_GW^T + \sigma^2 I) W\bv{\hat z}_i + c_i (WK_GW^T + \sigma^2 I) z_0\\
  &= W(K_G W^T W + \sigma^2 I) \bv{\hat z}_i  + c_i W(K_GW^T z_0) + c_i \cdot \sigma^2 \cdot z_0\\
  &= W \left [(K_G W^T W + \sigma^2 I) \bv{\hat z}_i + c_iK_GW^T z_0 \right ] + c_i \cdot \sigma^2 \cdot z_0.
  \end{align*}
  which completes the derivation of the update.

As discussed in Sec.~\ref{sec:gsgpIter}, it takes $\O(n + m \log m)$ time to precompute $K_G W^T \bv{z}_0$, after which: it takes $\O(m \log m)$ time to compute $(K_G W^T W + \sigma^2 I) \bv{\hat z}_{i}$, it takes $\mathcal O(m)$ time to add in $ c_i K_G W^T \bv{z}_0$, and it takes $\mathcal{O}(1)$ time to update $c_i$, for a total of $\O(m \log m)$ time for each update.

 \end{proof}
 
 \begin{repclaim}{clm:goldenRule2}
For any $\bv{z}_i,\bv{y}_i \in \R^n$ with $\bv{z}_i = W \bv{\hat z}_i + c_{i} \bv{z}_0$ and $\bv{y}_i = W \bv{\hat y}_i + d_{i} \bv{y}_0$,
\begin{align*}
\bv{z}_i^T \bv{y}_i = \bv{\hat z}_i^T W^T W \bv{\hat y}_i + d_{i} \bv{\hat z}_i^T W^T  \bv{y}_0 + c_{i} \bv{\hat y}_i^T W^T  \bv{z}_0 +   c_{i}  d_i  \bv{y}_0^T  \bv{z}_0.
\end{align*}
We denote the above operation by $\langle (\bv{\hat z}_i, c_{i}), (\bv{\hat y}_i, d_{i}) \rangle$.
\end{repclaim}
\begin{proof}
We have:
\begin{align*}
\bv{z}_i^T \bv{y}_i = (W \bv{\hat z}_i + c_{i} \bv{z}_0)^T (W \bv{\hat y}_i + d_{i} \bv{y}_0) = \bv{\hat z}_i^T W^T W \bv{\hat y}_i + d_{i} \bv{\hat z}_i^T W^T  \bv{y}_0 + c_{i} \bv{\hat y}_i^T W^T  \bv{z}_0 +   c_{i}  d_i  \bv{y}_0^T  \bv{z}_0,
\end{align*}
giving the claim.
\end{proof}

\begin{algorithm}
\caption{Efficiently Factorized Conjugate Gradient  (EFCG)}
\label{alg:efcg_2}
\begin{algorithmic}[1]
\Procedure {EFCG}{$K_G, W, \bv{b}, \sigma, \bv{x}_0, \epsilon $} 
\State \textit{New Iterates:} 
\State \hspace{1cm} $\mathbf{\hat{v}}_k = \left[ K_GW^TW + \sigma^2I \right] \mathbf{\hat{r}}_k$,$\mathbf{\hat{u}}_k = W^TW \mathbf{\hat{r}}_k$
\State \hspace{1cm}  $\mathbf{\hat{z}}_k = \left[ K_GW^TW + \sigma^2I \right] \mathbf{\hat{p}}_k$, $\mathbf{\hat{s}}_k = W^TW \mathbf{\hat{p}}_k$
\State $\bv{r_0} = \bv{b} - \tilde K \bv{x}_0,\, \bv{\hat r_0} = \bv{0},\, c_{0}^r = 1$
\State $\mathbf{\hat{p}}_{0} = \bv{0},\, c_{0}^p = 1$, $\, \bv{\hat x}_0 = \bv{0},\, c_{0}^x = 0$
\State $ \mathbf{\hat{v}}_0  = \mathbf{0}, \mathbf{\hat{u}}_0  = \mathbf{0}$
\State $ \mathbf{\hat{z}}_0  = \mathbf{0}, \mathbf{\hat{s}}_0  = \mathbf{0}$
\For{$k = 0$ to maxiter} 
\State $\alpha_{k} =  \frac{\mathbf{\hat{u}}_{k}^{T} \mathbf{\hat{r}}_{k} +  c_{k}^r c_{k}^r \norm{\bv{r}_0}  + 2  c_{i}^r  (\mathbf{\hat{r}}_{i}^{T}  W^T\bv{r}_0 )}{ \mathbf{\hat{s}}_{k}^{T}\mathbf{\hat{z}}_{k}  + c_{k}^{z} c_{k}^{p}  \norm{\bv{r}_0}   + c_{k}^{p}  (\mathbf{\hat{z}}_{k}^{T}  W^T\bv{r}_0 ) +  c_{k}^{z}  ( \mathbf{\hat{r}}_{k}^{T} W^T\bv{r}_0 )}$
\State $(\mathbf{\hat x}_{k+1}, c_{k+1}^x) = (\mathbf{\hat x}_{k}, c_{k}^x)+  \alpha_{k} \cdot ( \mathbf{\hat{p}}_k, c_{k}^p)$ 
\State $(\mathbf{\hat{r}}_{k+1},c_{k+1}^r) = (\mathbf{\hat{r}}_{k},c_{k}^r)-  \alpha _{k} \cdot ( \mathbf{\hat{z}}_k + c_{k}^p \cdot (K_G W^T \bv{r}_0)  , \sigma^2  c_{k}^p)$  
\State $[\mathbf{\hat{v}}_{k+1}, \mathbf{\hat{u}}_{k+1}]  = \mathcal{B}( \mathbf{\hat{r}}_{k+1})$
\State  if $\mathbf{\hat{u}}_{k+1}^{T} \mathbf{\hat{r}}_{k+1} +  c_{k+1}^r c_{k+1}^r  \norm{\bv{r}_0} + 2  c_{k+1}^r (\mathbf{\hat{r}}_{k+1}^{T} W^T\bv{r}_0) \leq \epsilon $ exit loop 
\State $\beta_{k}  =  \frac{\mathbf{\hat{u}}_{k+1}^{T} \mathbf{\hat{r}}_{k+1} +  c_{k+1}^r c_{k+1}^r + 2  c_{k+1}^r (\mathbf{\hat{r}}_{k+1}^{T}  W^T\bv{r}_0) }{\mathbf{\hat{u}}_{k}^{T} \mathbf{\hat{r}}_{k} +  c_{k}^r c_{k}^r \norm{\bv{r}}_0 + 2  c_{k}^r (\mathbf{\hat{r}}_{k}^{T}W^T\bv{r}_0  )} $ 
\State $(\mathbf{\hat{p}}_{k+1},c_{k+1}^p) = (\mathbf{\hat{r}}_{k+1},c_{k+1}^r) + \beta_{k} \cdot (\mathbf{\hat{p}}_{k},c_{k}^p) $
\State $\mathbf{\hat{s}}_{k+1} = \mathbf{\hat{u}}_{k+1} + \beta_{k} \cdot \mathbf{\hat{s}}_{k} $
\State $\mathbf{\hat{z}}_{k+1} = \mathbf{\hat{v}}_{k+1} + \beta_{k} \cdot \mathbf{\hat{z}}_{k} $
\EndFor
\State \textbf{return} $ \bv{x}_{k+1} = W \bv{ \hat x}_{k+1} + c_{k+1}^x \cdot \bv{r}_0 + \bv{x}_0$
\EndProcedure
\end{algorithmic}
\end{algorithm}

\subsection{Efficiently Factorized Conjugate Gradient Algorithms}
We now present Efficiently factorized conjugate gradient (EFCG) (Algorithm \ref{alg:efcg_2}), which improves on our basic Factorized CG algorithm (Algorithm \ref{alg:fcg}) by avoiding any extra multiplication with $W$ and $W^TW$. The central idea of EFCG is to exploit the fact that each time we perform a matrix-vector multiplication with the GSGP operator $ K_GW^TW + \sigma^2I$, we also must perform one with $W^TW$. We can save the result of this multiplication to avoid repeated work.  In particular, this lets us avoid extra  MVM costs associated with $W^TW$ present in factorized inner products steps of Algorithm \ref{alg:fcg}. Similar to Algorithm \ref{alg:fcg}, Algorithm \ref{alg:efcg_2} also maintains iterates CG (Algorithm \ref{alg:cg}) exactly in the same compressed form $\bv{x}_k = W \bv{\hat x}_k + c_{k}^x \bv{r}_0 + \bv{x}_0$, $\bv{r}_k = W \bv{\hat r}_k + c_{k}^r \bv{r}_0$, and $\bv{p}_k = W \bv{\hat p}_k + c_{k}^p \bv{r}_0$.

We let $[\bv v,\bv u] =\mathcal{B}(\bv x)$ denote the operation that returns $\bv v = \left[ K_GW^TW + \sigma^2I \right] \bv x$ and $\bv u = W^T W \bv x$. Since $\bv u$ must be computed as in intermediate step in computing $\bv v$, this operation has the same cost as a standard matrix-vector-multiplicaiton with $\left[ K_GW^TW + \sigma^2I \right]$. Notice that Algorithm \ref{alg:efcg_2} performs just  one $\mathcal{B}(\bv{x})$ operation per iteration, requiring a single matrix vector multiplication with each of $K_G$ and $W^T W$ per iteration. Both $K_G W^T \bv{r}_0$ and $W^T \bv{r}_0$ are precomputed.
 
 
In addition to $\mathcal{B}(\bv{x})$  operation, the superior efficiency of the EFCG Algorithm \ref{alg:efcg_2} over CG and Factorized CG can mainly be attributed to following facts: 

 \begin{itemize}
%
%
\item  It uses four new iterates:  $\mathbf{\hat{v}}_k = \left[ K_GW^TW + \sigma^2I \right] \mathbf{\hat{r}}_k$,     $\mathbf{\hat{u}}_k = W^TW \mathbf{\hat{r}}_k$,   $\mathbf{\hat{z}}_k = \left[ K_GW^TW + \sigma^2I \right] \mathbf{\hat{p}}_k$ and  $\mathbf{\hat{s}}_k = W^TW \mathbf{\hat{p}}_k$. Given these iterates all factorized inner products can be computed without any extra multiplication with  $W^TW$.
 \item In case, the initial solution $\bv{x}_0 = \bv{0}$, which is the most common choice in practice for CG,  $\tilde{K} \bv{x}_0$  multiplication can be avoided.  Also, observe that in SKI mean and covariance approximation (Definition \ref{def:kissgp}),  we only need  $W^T \bv{x}_{k+1}$ which is equal  to $W^TW \bv{\hat x}_{k+1} + W^T\bv{r}_0 + W^T\bv{x}_0$. Since,  $W^T\bv{r}_0$ is pre-computed and $W^T\bv{x}_0 =0$, no extra multiplication with $W$ or $W^TW$ is required other than computing  $K_G W^T \bv{r}_0$ and $W^T \bv{r}_0$.
 \end{itemize} 
 
 In Algorithm \ref{alg:efcg}, we present a further simplified variant on EFCG for the case when initial residual $\bv{r}_0$ is in the span of $W$ to directly compute $W^T\bv{x}_{k+1}$ where $\bv{x}_{k+1}$ is the final solution returned by the EFCG algorithm. Observe SKI mean and covariance expressions of SKI definition (i.e., Definition \ref{def:kissgp}), we always need to post-process $\bv{x}_{k+1}$ as $W^T\bv{x}_{k+1}$ to estimate them. Unlike EFCG, simplified EFCG (i.e., Algorithm \ref{alg:efcg}) maintains a compressed form for $\bv{r}_k$ using $\bv{\hat r}_k$ (as  $\bv{r}_k = W \bv{\hat r}_k$) and doesn't maintain $\bv{p}_k$. In addition to that, Algorithm \ref{alg:efcg}  also maintains another iterate $\bv{\hat x}^d_{k+1}$ such that $W^T\bv{x}_{k+1} =  \bv{\hat x}^d_{k+1} + W^T\bv{x}_{0}$. 
 
 \begin{algorithm}
\caption{Simplified EFCG -- Initial residual (i.e., $\bv{r}_0$) is in span of $W$}
\label{alg:efcg}
\begin{algorithmic}[1]
\Procedure {EFCG}{$K_G, W,\bv{\hat r}_0 , \sigma, \epsilon$} 
\State \textit{New Iterates:} 
\State \hspace{1cm} $\mathbf{\hat{v}}_k = \left[ K_GW^TW + \sigma^2I \right] \mathbf{\hat{r}}_k$,$\mathbf{\hat{u}}_k = W^TW \mathbf{\hat{r}}_k$
\State \hspace{1cm}  $\mathbf{\hat{z}}_k = \left[ K_GW^TW + \sigma^2I \right] \mathbf{\hat{p}}_k$, $\mathbf{\hat{s}}_k = W^TW \mathbf{\hat{p}}_k$
\State $  \mathbf{\hat x}^{d}_{0}   = \mathbf{0}$
\State $ \mathbf{\hat{v}}_{0},   \mathbf{\hat{u}}_{0}   = \mathcal{B}(\mathbf{\hat{r}}_0)$
\State $ \mathbf{\hat{z}}_{0}  = \mathbf{\hat{v}}_{0}, \mathbf{\hat{s}}_{0}  =\mathbf{\hat{u}}_{0}$ 
\For{$k = 1$ to $maxiter$} 
\State $\alpha_{k} = {\frac{\mathbf{\hat{u}}_{k}^T  \mathbf{\hat{r}}_{k}}{\mathbf{\hat{s}}_k^T  \mathbf{\hat{z}}_k}} $ 
\State $ \mathbf{\hat x}^{d}_{k+1} = \mathbf{\hat x}^{d}_{k} + \alpha_{k} \cdot \mathbf{\hat{s}}_k $ 
\State $\mathbf{\hat{r}}_{k+1} = \mathbf{\hat{r}}_k - \alpha _{k} \cdot \mathbf{\hat{z}}_k $ 
\State $ \mathbf{\hat{v}}_{k+1}, \mathbf{\hat{u}}_{k+1}  = \mathcal{B}( \mathbf{\hat{r}}_{k+1})$
\State  if $\mathbf{\hat{u}}_{k+1}^{T} \mathbf{\hat{r}}_{k+1} \leq \epsilon $ exit loop 
\State $\beta_{k}  =  \frac{\mathbf{\hat{u}}_{k+1}^{T} \mathbf{\hat{r}}_{k+1}}{\mathbf{\hat{u}}_{k}^{T} \mathbf{\hat{r}}_{k}} $ 
\State $\mathbf{\hat{s}}_{k+1} = \mathbf{\hat{u}}_{k+1} + \beta_{k} \cdot \mathbf{\hat{s}}_{k} $
\State $\mathbf{\hat{z}}_{k+1} = \mathbf{\hat{v}}_{k+1} + \beta_{k} \cdot \mathbf{\hat{z}}_{k} $
\EndFor
\State \textbf{return}  $\mathbf{\hat x}^{d}_{k+1} $
\EndProcedure
\end{algorithmic}
\end{algorithm}

The requirement of initial residual $\bv{r}_0$ to be in the span of $W$ can be met in two ways. For example, for SKI posterior mean inference, we set  $\mathbf{{x}}_0= \frac{1}{\sigma^2} \cdot \mathbf{y}$ implying $\mathbf{r}_0 = \mathbf{y} -  \left[ WK_GW^T + \sigma^2I\right] \frac{1}{\sigma^2} \cdot \mathbf{y} = - \frac{1}{\sigma^2} \cdot  W K_GW^T \mathbf{y}$, which lies in the span of $W$.  Consequently, we initiate Algorithm \ref{alg:efcg} with $\mathbf{\hat r}_0 = -  \frac{1}{\sigma^2} \cdot K_GW^T \mathbf{y} \in \R^{m \times 1}$ and following the invariance of simplified EFCG algorithm, compute transformed solutions as  $W^T\bv{x}_{k+1} =  \bv{\hat x}^d_{k+1} +  \frac{1}{\sigma^2} \cdot  W^Ty$. Notice that simplified EFCG (unlike EFCG) does not require pre-computations of terms $K_G W^T \bv{r}_0$ and $W^T \bv{r}_0$. In fact, simplified ECFG requires only only multiplication with $W^T$, i.e., to obtain $W^Ty$ which is sufficient for both, initialization of $\mathbf{r}_0$ and  the transformed final solution $W^T\bv{x}_{k+1}$ (which is equal to $W^T\left( \tilde{K}_X + \sigma^2 I \right)^{-1} \bv{y}$). Hence, simplified EFCG can compute SKI posterior mean using only sufficient statistics ($W^TW$ and $W^T\bv{y}$). 
 
A second way to meet the requirement of the initial residual $\bv{r}_0$ being in the span of $W$, is by setting $\mathbf{{x}}_0 = \bv{0}$  when $\bv{\hat{y}}$ is provided such that $\bv{y} = W \bv{\hat y}$. This is the case, e.g., in posterior covariance approximation. The final transformed solution $W^T\bv{x}_{k+1}$ in this setting reduces to $\bv{\hat x}^d_{k+1}$. Notice for this setting also, we need only $W^TW$  and do not require the matrix $W$. 

Consquently, simplified EFCG can compute both SKI posterior mean and covariance function using onlythe sufficient statistics ($W^TW$ and $W^T\bv{y}$) and without even realizing $W$ matrix in memory. 


\subsection{Factorized Lanczos Algorithms}
The Lanczos algorithm can be utilized to factorize a symmetric matrix $A \in \R^{n \times n }$ as $QTQ^{T}$ such that $T \in \R^{n \times n}$ is a symmetric tridiagonal matrix and $Q \in \R^{n \times n}$ is orthonormal matrix. Previously, it has been used with $k$ iterations (i.e. Algorithm \ref{alg:lanczos}) to compute low-rank and fast approximations of SKI covariance matrix and log-likelihood of the data. For further details, we refer readers to  \cite{ConstantTimePD, ScalableLD}.

\begin{algorithm}
\caption{Lanczos Algorithm (LA)}
\label{alg:lanczos}
\begin{algorithmic}[1]
\Procedure {LA}{$K_G, W, \bv{b}, \sigma, k$} 
\State $\mathbf{q}_{0} = \mathbf{0}, \mathbf{q}_{1} = \mathbf{b}, \beta_1 = 0$ 
\State $Q_{:, 1} =\mathbf{q}_{1} $ 
\For{$i = 1$ to $k$} 
    \State $\mathbf{q}_{i+1} = \tilde{K} \mathbf{q}_{i} - \beta_{i} \cdot \mathbf{q}_{i-1}$
    \State $\alpha_i = \mathbf{q}_{i}^{T}\mathbf{q}_{i+1}$
    \State $T_{i,i} = \alpha_i$
    \State  if $i  == k$ then exit loop 
    \State $\mathbf{q}_{i+1} = \mathbf{q}_{i} - \alpha_{i} \cdot  \mathbf{q}_{i}$
    \State $\mathbf{q}_{i+1} = \mathbf{q}_{i+1} - \left[Q_{:, 1}, ..., Q_{:, i} \right] \left( \left[Q_{:, 1}, ..., Q_{:, i} \right]^{T} \mathbf{q}_{i+1} \right) $
    \State $\beta_{i+1} = \norm{\mathbf{q}_{i+1}}$
    \State $T_{i, i+1} = T_{i, i+1} = \beta_{i+1}$
    \State $\mathbf{q}_{i+1} = \frac{1}{\beta_{i+1}}  \cdot    \mathbf{q}_{i+1}  $     
    \State $Q_{:, i+1} =\mathbf{q}_{i+1} $ 
\EndFor
\State \textbf{return} $Q$, $T$
\EndProcedure
\end{algorithmic}
\end{algorithm}


\begin{algorithm}
\caption{Factorized Lanczos Algorithm (FLA)}
\label{alg:factorized_lanczos}
\begin{algorithmic}[1]
\Procedure{FLA}{$K_G, W, \bv{b}, \sigma, k$} 
\State $\mathbf{\hat{q}}_{0} = \mathbf{0}, c_0^{q} = 0,  \mathbf{\hat{q}}_{1} = \mathbf{0}, c_1^{q} = 1, \beta_1 = 0$ 
\State $\hat{Q}_{:, 1} = \mathbf{\hat{q}}_{1}, \mathbf{\Lambda} = \mathbf{0} \in \R^{k \times 1}, \bv{d} = \mathbf{0} \in \R^{k \times 1} $
\For{$i = 1$ to $k$} 
  \State $(\mathbf{\hat{q}}_{i+1} ,c_{i+1}^q) = (\mathbf{\hat{q}}_{i} ,c_{i}^q)-  \beta_{i} \cdot \mathcal{A}( \mathbf{\hat{q}}_{i-1} , c_{i-1}^q)$ 
    \State $\alpha_{i} 
=\langle (\mathbf{\hat{q}}_{i} ,c_{i}^q),  (\mathbf{\hat{q}}_{i+1} ,c_{i+1}^q)\rangle $
    \State $T_{i,i} = \alpha_i;$ \hskip3mm $\mathbf{d}_{i} = c_i^{q}$
    \State  if $i  == k$ then exit loop
     \State $(\mathbf{\hat{q}}_{i+1} ,c_{i+1}^q) = (\mathbf{\hat{q}}_{i} ,c_{i}^q)-  \alpha_{i} \cdot ( \mathbf{\hat{q}}_{i} , c_{i}^q)$ 
    \State $ \mathbf{\Lambda}_{j} = \langle (\hat{Q}_{:,j} , c_{j}^q), (\mathbf{\hat{q}}_{i+1} ,c_{i+1}^q) \rangle;$ \hskip3mm $\forall j \in \{1,...,i\}$
    \State $\mathbf{\hat{q}}_{i+1} = \mathbf{\hat{q}}_{i+1} -  \hat{Q}_{:, 1:i}  \mathbf{\Lambda}_{1:i};$ \hskip3mm $c_{i+1}^{q} = c_{i+1}^{q} - \mathbf{d}_{1:i}^{T} \mathbf{\Lambda}_{1:i}$
    \State $\beta_{i+1} = \sqrt{\langle (\mathbf{\hat{q}}_{i+1} ,c_{i+1}^q) , (\mathbf{\hat{q}}_{i+1} ,c_{i+1}^q) \rangle}$
    \State $T_{i, i+1} = T_{i, i+1} = \beta_{i+1}$
    \State $\mathbf{\hat{q}}_{i+1} = \frac{1}{\beta_{i+1}}  \cdot \mathbf{\hat{q}}_{i+1};$ \hskip3mm $c_{i+1}^{q} = \frac{c_{i+1}^{q}}{\beta_{i+1}}$
    \State $\hat{Q}_{:, i+1} =\mathbf{\hat{q}}_{i+1} $ 
\EndFor
\State \textbf{return} $\hat{Q}$, $T$, $\mathbf{d}$
\EndProcedure
\end{algorithmic}
\end{algorithm}

\begin{algorithm}
\caption{Efficiently Factorized Lanczos algorithm (EFLA)}
\label{alg:efficiently_factorized_lanczos}
\begin{algorithmic}[1]
\Procedure{EFLA}{$K_G, W, \bv{b}, \sigma, k$} 
\State \textit{New Iterates:} 
\State \hspace{1cm}   $\hat{P}_{:, i} = W^TW\hat{Q}_{:, i}$ and $\bv{\hat s}_{i} = \left[K_G W^TW + \sigma I  \right] \mathbf{\hat{q}}_i$
\State $\mathbf{\hat{q}}_{0} = \mathbf{0}, c_0^{q} = 0,  \mathbf{\hat{q}}_{1} = \mathbf{0}, c_1^{q} = 1, \beta_1 = 0$ 
\State $\hat{Q}_{:, 1:k} = [\mathbf{0}, ...,\mathbf{0}] \in \R^{m \times k}, \hat{P}_{:, 1:k} = [\mathbf{0}, ...,\mathbf{0}] \in \R^{m \times k}$ 
\State $\hat{Q}_{:, 1} = \mathbf{\hat{q}}_{1}, \mathbf{\Lambda} = \mathbf{0} \in \R^{k \times 1}, \bv{\hat s}_{i} = \mathbf{0} \in \R^{m \times 1}, \bv{d} = \mathbf{0} \in \R^{k \times 1}  $
\For{$i = 1$ to $k$} 
    \State $\mathbf{\hat{q}}_{i+1} = \bv{\hat s}_{i}  + c_i^{q} \cdot K_GW^T\mathbf{b} - \beta_{i} \cdot \mathbf{\hat{q}}_{i-1}$
    \State $c_{i+1}^{q} = \sigma^2 c_i^{q} - \beta_i c_{i-1}^{q}$
   \State $\alpha_i = \hat{P}_{:, i}^T\mathbf{\hat{q}}_{i+1} + c_i^{q} c_{i+1}^{q} + \left(c_i^{q} \cdot \mathbf{\hat{q}}_{i+1} + c_{i+1}^{q} \cdot \mathbf{\hat{q}}_{i} \right)^{T} W^T\mathbf{b} $
    \State $T_{i,i} = \alpha_i;$ \hskip3mm $\mathbf{d}_{i} = c_i$
    \State  if $i  == k$ then exit loop
        \State $(\mathbf{\hat{q}}_{i+1} ,c_{i+1}^q) = (\mathbf{\hat{q}}_{i} ,c_{i}^q)-  \alpha_{i} \cdot ( \mathbf{\hat{q}}_{i} , c_{i}^q)$ 
    \State $ \mathbf{\Lambda}_{1:i} = \hat{P}_{:, 1:i}^T \mathbf{\hat{q}}_{i+1} + \left(c_{i+1}^{q} + \mathbf{\hat{q}}_{i+1}^{T}W^T\mathbf{b} \right) \cdot \mathbf{c}_{1:i}  + c_{i+1}^{q} \cdot \hat{Q}_{:, 1:i}^TW^T\mathbf{b}$
    \State $\mathbf{\hat{q}}_{i+1} = \mathbf{\hat{q}}_{i+1} - \hat{Q}_{:, 1:i} \mathbf{\Lambda}_{1:i};$ \hskip3mm $c_{i+1}^{q} = c_{i+1}^{q} - \mathbf{d}_{1:i}^{T} \mathbf{\Lambda}_{1:i}$
    \State $\bv{\hat s}_{i+1} , \hat{P}_{:, i+1} = \mathcal{B} (\mathbf{\hat{q}}_{i+1})$
    \State $\beta_{i+1} = \sqrt{\hat{P}_{:, i+1}^T \mathbf{\hat{q}}_{i+1} + c_{i+1}^{q} c_{i+1}^{q} + 2 c_{i+1}^{q} \mathbf{\hat{q}}_{i+1}^{T}W^{T}\mathbf{b}}$
    \State $T_{i, i+1} = T_{i, i+1} = \beta_{i+1}$
     \State $\mathbf{\hat{q}}_{i+1} = \frac{1}{\beta_{i+1}}  \cdot \mathbf{\hat{q}}_{i+1};$ \hskip3mm $c_{i+1}^{q} = \frac{c_{i+1}^{q}}{\beta_{i+1}}$
    \State $\bv{\hat s}_{i+1}   =\frac{1}{\beta_{i+1}}  \cdot \bv{\hat s}_{i+1}  ;$ \hskip3mm $\hat{P}_{:, i+1} =  \frac{1}{\beta_{i+1}}  \cdot \hat{P}_{:, i+1}$
    \State $\hat{Q}_{:, i+1} =\mathbf{\hat{q}}_{i+1} $ 
\EndFor
\State \textbf{return} $\hat{Q}$, $T$, $\mathbf{d}$
\EndProcedure
\end{algorithmic}
\end{algorithm}

Similar to Factorized CG, we derive factorized Lanczos algorithm (FLA) using  factorized inner products and matrix vector multiplication, as described in Algorithm \ref{alg:factorized_lanczos}. We maintain all iterates in $\R^{m \times 1}$ similar to Factorized CG, in particular,  ${Q}_{:,i} \in \R^{n \times 1}$ vectors are maintained in compressed form such that $Q_{:,i} = W\hat{Q}_{:,i} + d_{i} \cdot \bv{b}$. The $T \in \R^{k \times k}$ matrix of Lanczos algorithm is retained as it is in FLA. Furthermore, in a manner similar to EFCG, we derive  efficient factorized algorithm (EFLA) as shown in Algorithm  \ref{alg:efficiently_factorized_lanczos}. Specifically, EFLA relies on two new iterates: $\hat{P}_{:, i} = W^TW\hat{Q}_{:, i}$ and $\bv{\hat s}_{i}  = \left[K_G W^TW + \sigma I  \right] \mathbf{\hat{q}}_i$ and maintains iterates of Lanczos algorithm as $Q_{:,i} = W\hat{Q}_{:,i} + d_{i} \cdot \bv{b}$, similar to FLA. Notice that EFLA only requires  one $\mathcal{B}(\bv{x})$ operation per loop thereby avoiding any extra MVMs with $W$ and $W^TW$, except one time pre-computations of $K_GW^T\bv{b}$ and $W^T\bv{b}$.


\section{Experiments -- Omitted Details and Additional Results}
\label{app:experiments_details_and_additional_results}

\subsection{Hardware and hyper-parameters details}
\label{app:hyper-parameters and dataset details}

We run all of our experiments on Intel Xeon Gold 6240 CPU @ 2.60GHz with 10 GB of RAM.  In all experimental settings, our kernels are squared exponential kernels wrapped within a scale kernel \cite{kissgp, ScalableLD}. Therefore, our hyper-parameters are $\sigma$, length-scales and output-scale as also presented in Table \ref{tab:hyper_parametetrs}. Length-scales are specific to each dimension for multi-dimensional datasets. For sine and sound datasets, we have utilized GPytorch to optimize hyper-parameters. For precipitation and radar datasets, we have considered previously optimized parameters in \citep{ScalableLD} and in \cite{angell2018inferring}, respectively. 

\begin{table}[H]
\centering
\begin{tabular}[t]{lccc}
\toprule
Dataset & $\sigma$ & Length-scale & Output-scale \\ \hline
Sine & 0.074 & 0.312 & 1.439  \\ \hline
Sound & 0.009 & 10.895 & 0.002 \\ \hline
Radar & 50.000 & [0.250, 0.250, 200] & 3.500 \\ \hline
Precipitation & 3.990 & [3.094, 2.030, 0.189]  & 2.786 \\ \hline
\end{tabular}
\caption{Hyper-parameters used for all datasets. Length-scale is of size $d$ of each dataset.}
\label{tab:hyper_parametetrs}
\end{table}

\subsection{Results: Synthetic sine dataset}

Figure \ref{fig:results_sine_dataset_additional} depicts the number of iteration and pre-processing time taken by GSGP and SKI for synthetic sine dataset wrt number of sample, for the setting on which Figure \ref{fig:sine_per_iteration_time_vs_n} reports the results. The number of iterations for GSGP and SKI are always close and possibly differ only due to finite precision. 

\begin{figure}[h]
    \centering \includegraphics[width=0.6  \textwidth]{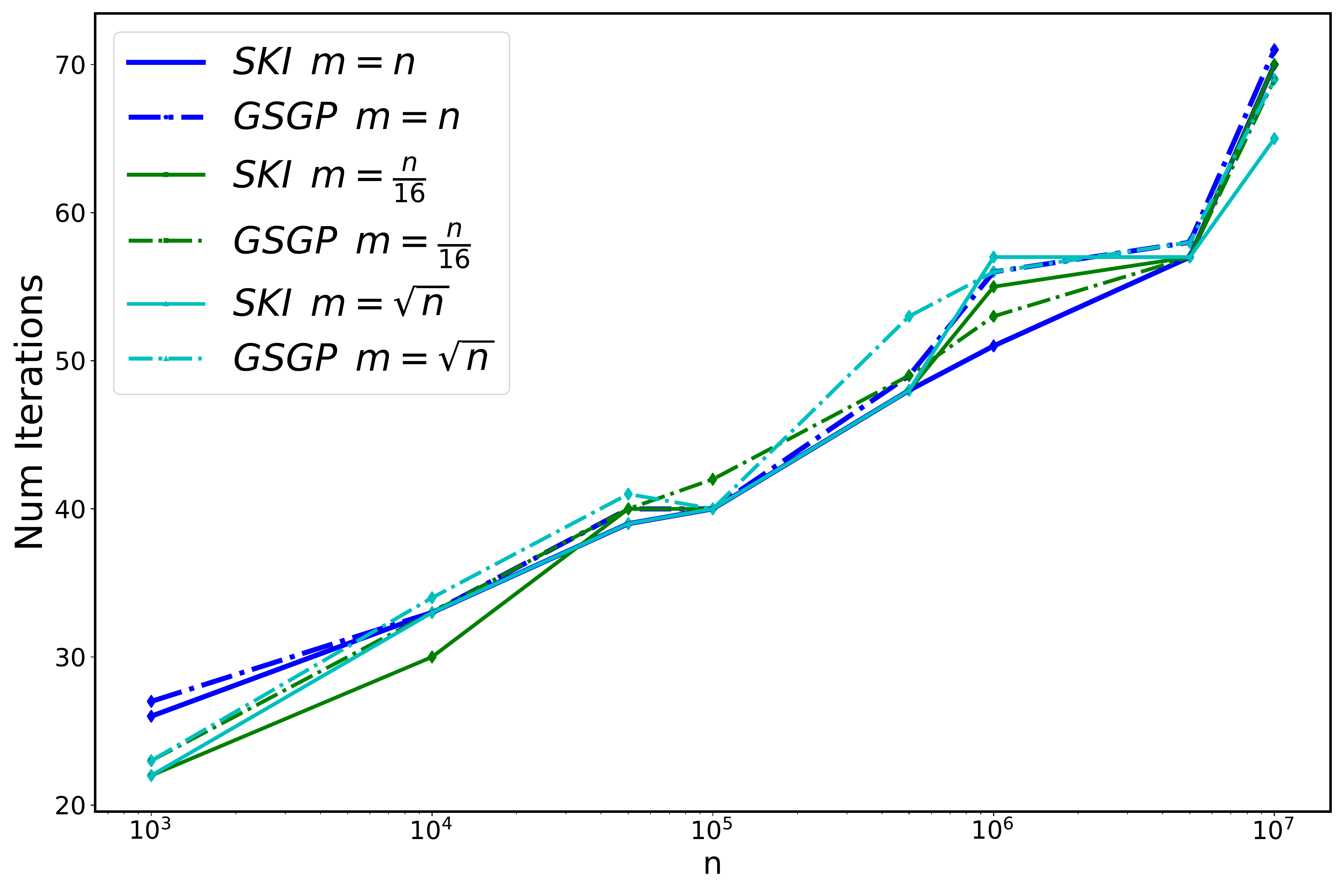}
    \caption{Number of iterations taken by SKI and GSGP on synthetic dataset. Results are averaged over 8 trials.}
    \label{fig:results_sine_dataset_additional}
\end{figure}

\subsection{Results: Precipitation dataset}

\begin{figure}[h]
    \begin{center}
    \begin{subfigure}{.38\textwidth}
    \includegraphics[width=\textwidth]{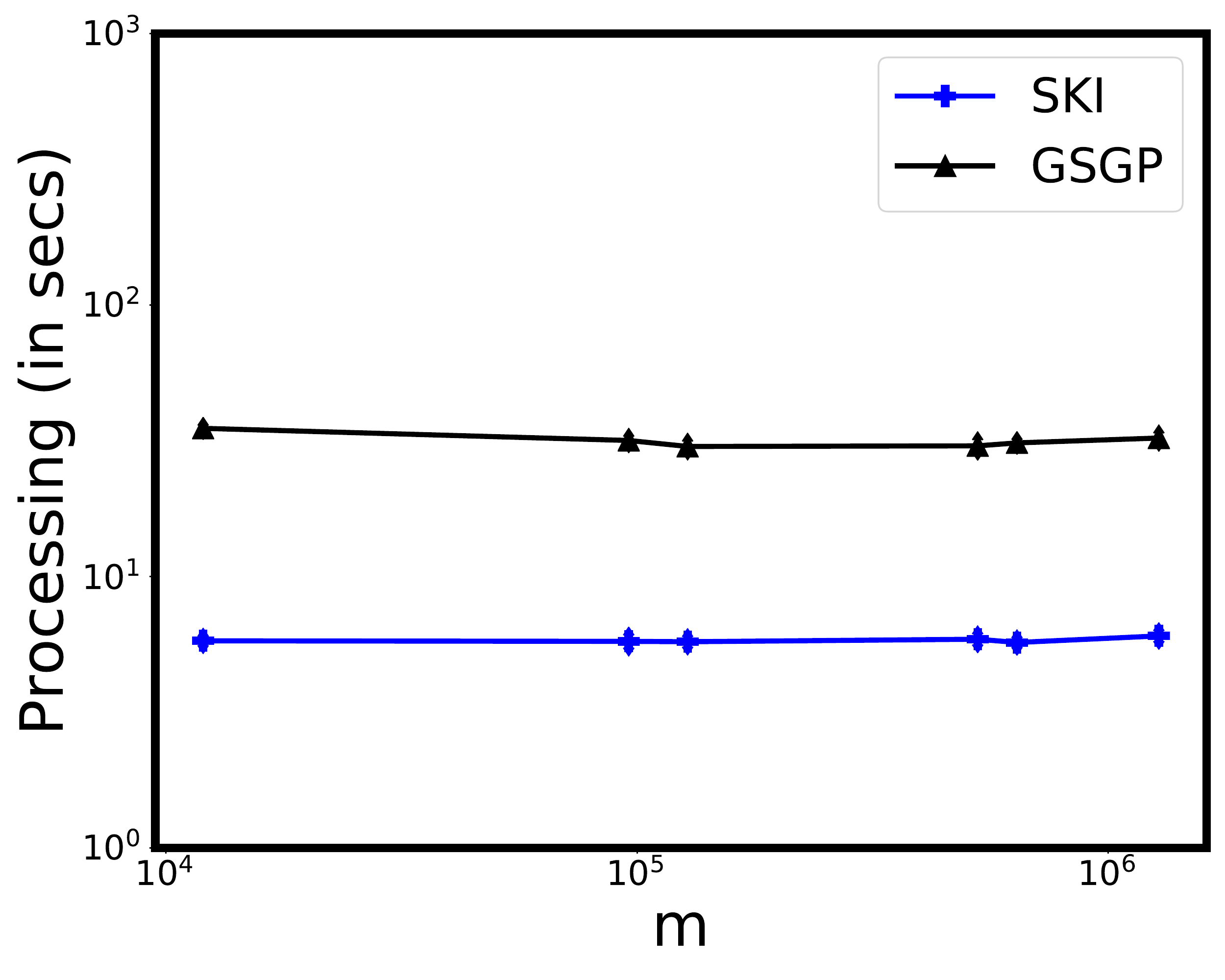}
    \end{subfigure}
    \begin{subfigure}{.38\textwidth}
    \includegraphics[width=\textwidth]{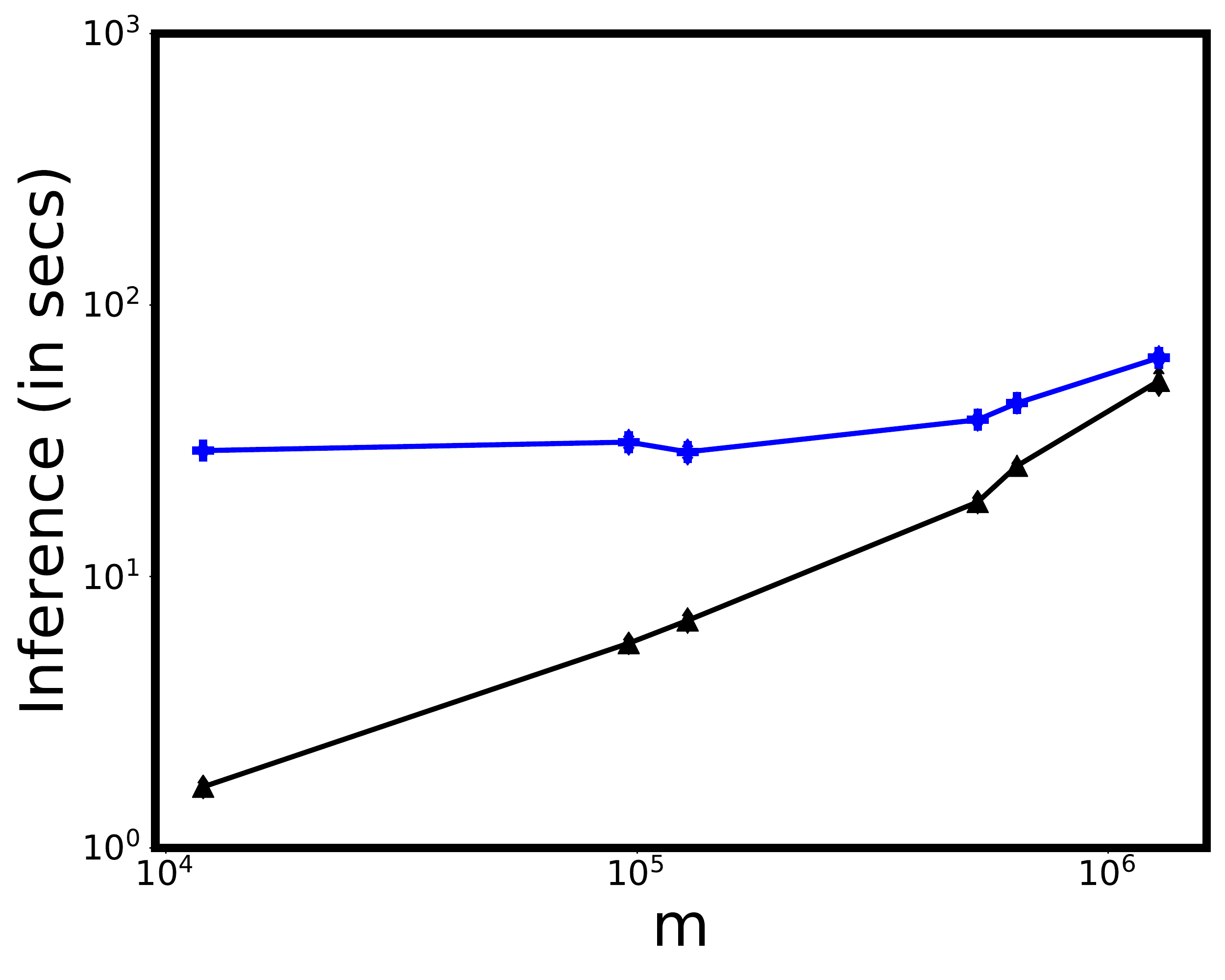}
    \end{subfigure} 
    \begin{subfigure}{.38\textwidth}
    \includegraphics[width=\textwidth]{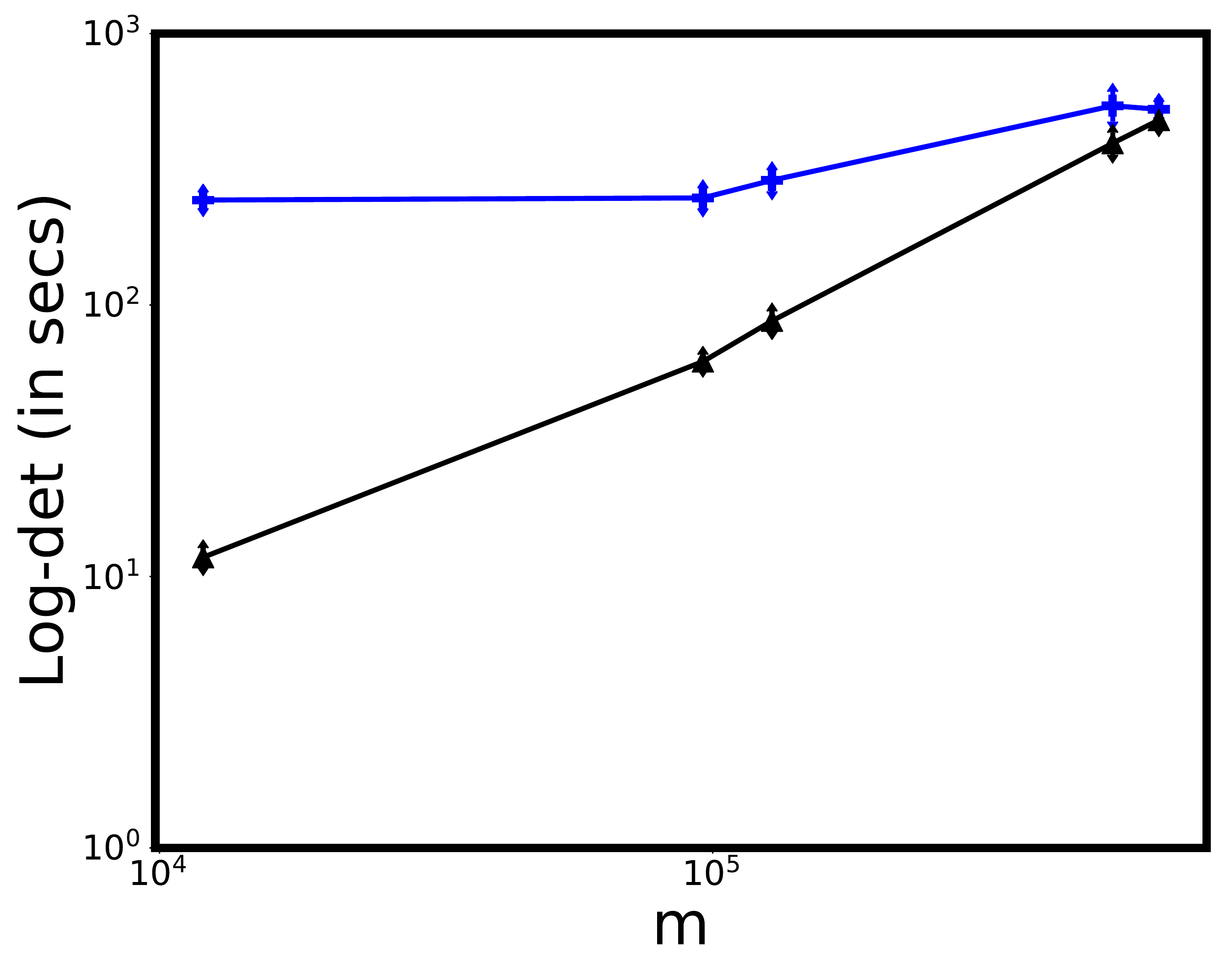} 
    \end{subfigure}
    \begin{subfigure}{.38\textwidth}
    \includegraphics[width=\textwidth]{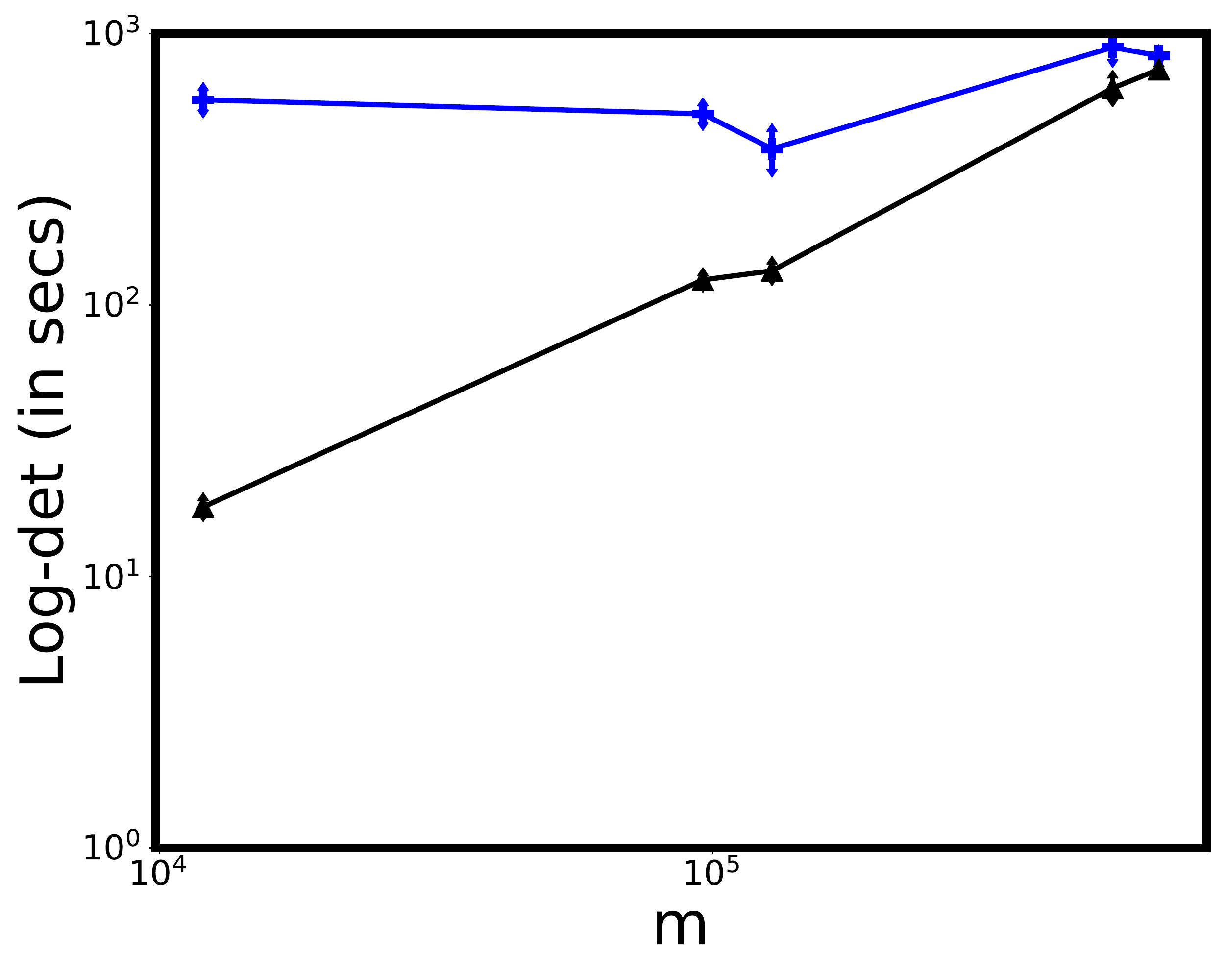} 
    \end{subfigure}
    \end{center}
    \caption{Inference time vs. $m$ for SKI inference tasks on precipitation data set. From left to right: pre-processing, mean inference, log-determinant for $tol=0.1$ and  for $tol=0.01$ and using 30 random vectors.}
    \label{fig:results_on_precip_dataset}
\end{figure}

Figure~\ref{fig:results_on_precip_dataset} shows running time vs. $m$ for GP inference tasks on precipitation data set of $n=528\K$ \cite{ScalableLD}. We consider $m \in \{12\K, 96\K, 128\K, 528\K,  640\K \}$. This is a situation where even for $m > n$, GSGP is faster compared to SKI. 
Pre-processing is up to 6x slower for GSGP due to the need to compute $W^TW$.
To perform only one mean inference, the overall time of GSGP and SKI \emph{including} pre-processing is similar as some of the per-iteration gains are offset by pre-processing. 
However, for the log-determinant computation task (as part of the log-likelihood computation), \emph{several} more iterations of linear solvers are required as also demonstrated by log-det computation in Figure \ref{fig:results_on_precip_dataset}. It is worth noting that pre-processing of GSGP is required only once which can be performed initially for the log-det computation and later be utilized for the posterior mean and covariance inference. Therefore, overall, GSGP is more effective than SKI for all inference tasks. 

\end{appendices}

\end{document}